\documentclass{article}

\usepackage{iclr2027_conference,times}
\iclrfinalcopy
\usepackage{amsmath,amssymb,amsthm,bm}
\usepackage{graphicx}
\usepackage{placeins}
\usepackage{subcaption}
\usepackage{booktabs}
\usepackage{tabularx}
\usepackage{multirow}
\usepackage{array}
\usepackage{microtype}
\usepackage{xcolor}
\usepackage{enumitem}
\usepackage{xurl}
\usepackage{hyperref}
\usepackage[nameinlink,capitalize,noabbrev]{cleveref}
\hypersetup{colorlinks=true,citecolor=blue!60!black,linkcolor=blue!60!black,urlcolor=blue!60!black}
\setlist{nosep,leftmargin=*}
\graphicspath{{figures/}}

\newcommand{\E}{\mathbb{E}}

\newtheorem{theorem}{Theorem}
\title{Terminal Shrinkage Averaging Reveals a Schedule-Estimator Interaction in LLM Pretraining}
\author{%
  Adam Ousherovitch \& Yixin Wang\\
  Department of Statistics, University of Michigan\\
  \href{mailto:aoushero@umich.edu}{\texttt{aoushero@umich.edu}}\qquad
  \href{mailto:yixinw@umich.edu}{\texttt{yixinw@umich.edu}}%
}
\hypersetup{pdfauthor={Adam Ousherovitch and Yixin Wang}}
\date{}

\begin{document}
\maketitle
\fancyhead{}
\vspace{-\baselineskip}

\begin{abstract}
Large language model (LLM) pretraining conventionally returns the raw final iterate. This couples two design choices: the learning-rate schedule that generates the parameter trajectory and the estimator that constructs the deployed model (e.g. the raw final iterate or a checkpoint average). A schedule that promotes optimization progress may differ from one that minimizes variation in the raw final iterate. Separating these choices creates an opportunity to maintain progress late in training while reducing variation in the returned model. To this end, we propose \emph{Terminal Shrinkage Averaging (TSA)}, which interpolates between the raw final iterate and the average of recent checkpoints to balance recent progress against terminal variation. We analyze how TSA changes the preferred terminal learning-rate schedule under a local quadratic approximation and test this interaction through a sequence of controlled NanoChat experiments. Finally, we demonstrate that the resulting gains transfer to depth-22 NanoChat, where the combined schedule and estimator improve validation quality. A qualifying time-to-GPT-2 run also finishes faster than the public baseline used in our experiments, providing preliminary evidence of benchmark acceleration.
\end{abstract}

\section{Introduction}

Large language model (LLM) pretraining typically follows a parameter trajectory under a decaying learning-rate schedule and ultimately deploys the raw final iterate \citep{brown2020language,hoffmann2022training,touvron2023llama2}. This convention answers two design questions with a single choice:
\emph{which trajectory should be followed near the end of training}, and \emph{which function of that trajectory should be deployed?}

The end of the schedule is usually a \emph{cooldown}, a terminal
reduction of the learning rate. Its length and specific implementation
materially change endpoint loss and can even reorder optimizer
comparisons
\citep{hoffmann2022training,bergsma2025straight,wen2025fantastic}. A
small terminal learning rate reduces the variation of the raw final
iterate, and it also limits how far the parameters can still move. A
trajectory that stays more active late in training can therefore keep
making progress while arriving at a worse raw final iterate.

\begin{figure}[h]
    \centering
\includegraphics[width=0.8\linewidth]{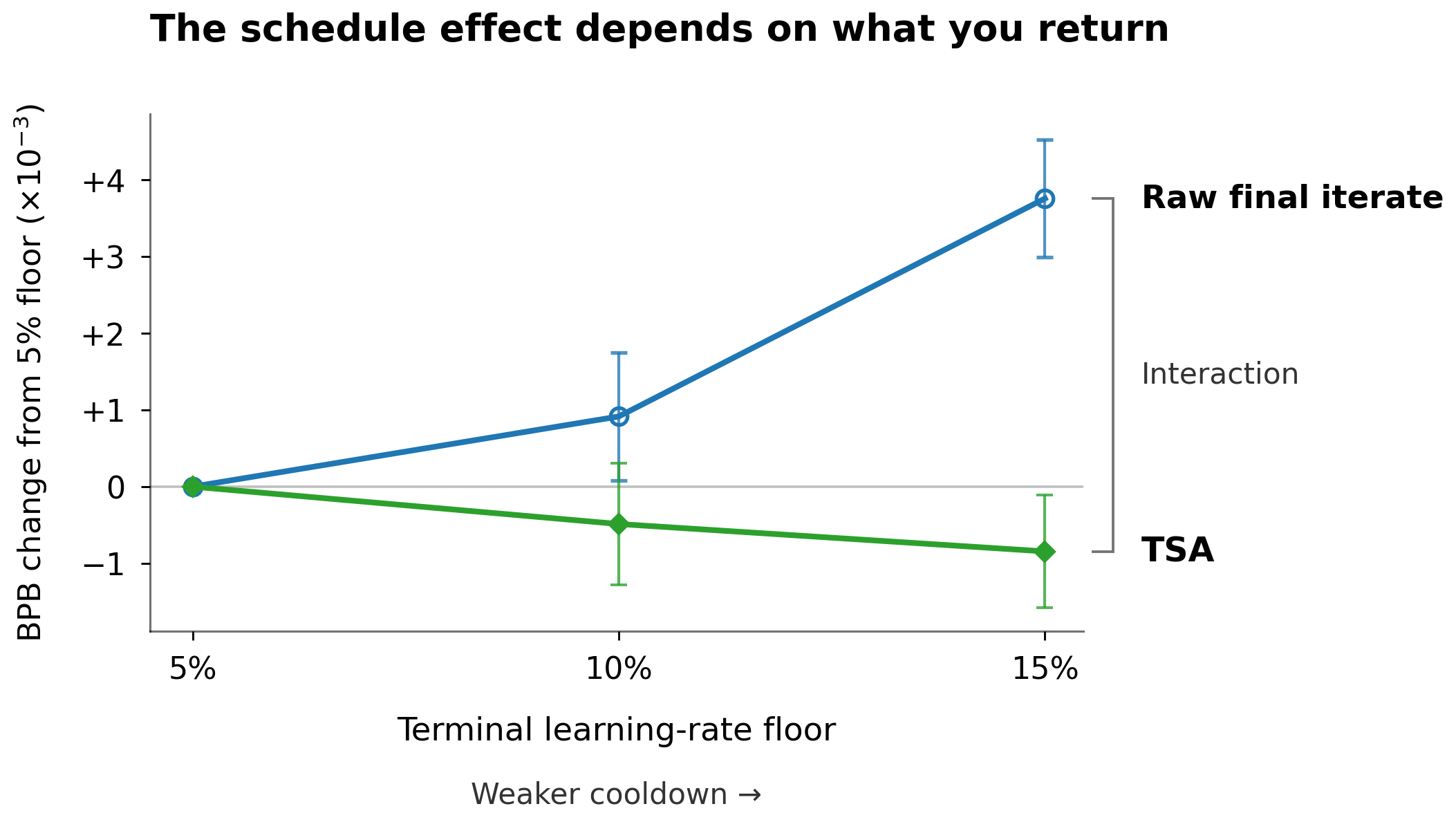}
    \caption{
\textbf{Averaging weights enables more aggressive terminal optimization.}
Raising the terminal learning-rate floor from $5\%$ to $15\%$ of peak worsens validation bits per byte (BPB) for the raw final iterate but improves it for TSA.}
    \label{fig:1a}
\end{figure}

Weight averaging separates the trajectory from the model returned from it \citep{polyak1992acceleration,izmailov2018averaging}. Recent checkpoint-averaging methods such as Latest Weight Averaging (LAWA) have also been used to accelerate pretraining \citep{kaddour2022stop}; in LLM pretraining, averaging can become more effective at higher learning rates \citep{sanyal2024early}. These findings motivate the
question we study: if the deployed model averages checkpoints, does a
weaker cooldown become preferable, and do the strength of the
averaging and the strength of the cooldown need to be chosen together?

To study these questions, we propose terminal shrinkage averaging
(TSA), an interpolation between the raw final iterate and the average of the
last $K$ checkpoints, with a single shrinkage coefficient setting how
much of the checkpoint average enters the returned model.

This interpolation allows us to have control over the strength of averaging. The raw final iterate and uniform LAWA are its two endpoints.
Under a local quadratic approximation of the loss, we derive its exact
surrogate risk, which is quadratic in the shrinkage coefficient. The
risk separates the displacement of the mean output from the reduction
in stochastic variation, demonstrating a bias-variance tradeoff, which
gives conditions under which the minimizer is interior, so that
neither the raw final iterate nor LAWA is optimal. We use the same
approximation to show that an averaged output prefers a more active terminal trajectory than the raw final iterate does.

Our experiments are designed to isolate the effect of averaging, the
effect of the terminal learning rate, and their interaction. Depth-12
NanoChat runs first show that the best shrinkage coefficient lies
strictly between the two endpoints. A paired factorial design then
crosses the terminal learning-rate floor ($5\%$, $10\%$, and $15\%$ of
the peak learning rate) with the output estimator (the raw final iterate
and TSA). Across five repetitions, raising the floor worsens
validation BPB when the raw final iterate is deployed and improves it when
TSA is deployed, and both effects grow with the size of the increase.
A replication under pure AdamW reproduces the direction of this
interaction, so this effect does not depend on NanoChat's Muon+AdamW optimizer. Finally, we apply the higher floor and TSA to the recipe
behind the best depth-22 NanoChat time-to-GPT-2 run, which gives
preliminary evidence that the interaction survives at scale, along
with a qualifying record run.

\paragraph{Contributions.} This work makes the following contributions.
\begin{itemize}
    \item We separate two choices in pretraining: the learning-rate
    schedule, which determines the parameter trajectory, and the \emph{output estimator}, which maps that trajectory to the deployed
    model. We introduce TSA, a shrinkage estimator between the raw final
    iterate and LAWA, and characterize its bias-variance tradeoff under a local quadratic approximation.

    \item Through controlled depth-12 NanoChat experiments, we identify that increasing terminal learning-rate activity is more favorable when TSA is returned, rather than the raw final iterate. Paired repetitions isolate this interaction from the effects of averaging or the terminal schedule alone.

    \item We demonstrate that the resulting gains transfer to the depth-22 NanoChat setting. At a matched training budget, the full schedule-estimator recipe improves validation quality over the underlying public recipe and produces a qualifying time-to-GPT-2 benchmark run.
\end{itemize}

\section{Related work}

\paragraph{Iterate and checkpoint averaging.}
Averaging has classical asymptotic guarantees in convex problems \citep{polyak1992acceleration,neu2018}, while finite-sample analyses of tail averaging show how averaging can reduce the variance of the final output \citep{jain2018parallelizing}. In deep learning, stochastic weight averaging (SWA) returns an average over a late training trajectory and often finds solutions with improved generalization \citep{izmailov2018averaging}. Latest Weight Averaging (LAWA) applies uniform averaging to a finite set of recent checkpoints and frames averaging as a way to accelerate training \citep{kaddour2022stop}.

For language-model pretraining, \citet{sanyal2024early} show that checkpoint averaging becomes more effective at higher learning rates, while \citet{ajroldi2025} characterize when averaging improves large-scale training. These results ask how the benefit of averaging changes under a given training trajectory. Our focus is complementary: we ask whether committing to an averaged output changes which terminal trajectory should be trained in the first place.

Lookahead also interpolates between fast and slow parameter states \citep{zhang2019lookahead}. It performs this interpolation repeatedly during optimization and resets the fast weights after each inner loop. In contrast, TSA leaves the optimizer updates unchanged and performs a single terminal interpolation between the raw final iterate and a finite late-checkpoint average to construct the deployed model.

\paragraph{Learning-rate schedules designed with averaging.}
Several lines of work connecting sustained optimizer activity, cooldowns, and averaging. Warmup-stable-decay schedules retain a long constant-learning-rate phase before a terminal decay \citep{hu2024minicpm}. This paper doesn't touch on averaging. \citet{hagele2024scaling} study constant-rate trajectories with cooldowns and show that SWA can be used in place of cooldowns. By allowing joint tuning of cooldowns and averaging, our method takes their observation that averaging plays a similar role to cooldowns and operationalizes it to improve optimization. Schedule-Free optimization asks the question of how to perform pretraining without knowing the length of the training a priori and strongly relies on learning-rate scheduling and iterate averaging \citep{defazio2024}. Anytime pretraining falls in a similar regime \citep{meterez2026anytime}. Our work is different from these two because it is in the setting where you know the training length in advance.  \citet{tian2026wsm} more formally study the interchangeability of averaging and decay. Their Warmup-Stable and Merge framework uses weighted checkpoint merging to emulate decay from a stable-rate trajectory. However, they do not study the interpolation and interaction between these two techniques. \citet{au2026training} utilize averaging throughout training and develop PACE, which modifies the live AdamW trajectory by pulling parameters toward an exponential moving average. They do not study the interaction between the decay and the averaging, however. TSA leaves the live optimizer iterate unmodified and separately intervenes on terminal learning-rate activity. These are orthogonal directions.

Our setting is very controlled. We retain the underlying horizon-dependent pretraining recipe, alter only its terminal learning-rate activity, and separately vary the returned estimator. This design isolates whether the effect of the terminal schedule depends on which model is ultimately returned.

\paragraph{Pretraining optimizers.}
Adam and AdamW remain standard scalar-preconditioned optimizers \citep{kingma2015,loshchilov2019}. Matrix-preconditioned methods include Shampoo \citep{gupta2018shampoo}, SOAP \citep{vyas2024soap}, and Muon, which orthogonalizes momentum updates for matrix-valued parameters \citep{jordan2024muon,liu2025muon}. Optimizer rankings can depend on model scale, data ratio, tuning budget, and evaluation horizon \citep{wen2025fantastic,wen2026hyperball}. TSA does not alter the underlying optimizer update rule and is therefore complementary to optimizer design. Our primary experiments use NanoChat's native Muon+AdamW optimizer, and the appendix examines how the averaging gain is distributed across the corresponding parameter groups.

\paragraph{Compute-optimal pretraining and capability thresholds.}
Scaling-law studies relate model size, data, and compute \citep{kaplan2020,hoffmann2022training}. NanoChat provides a compact full-stack training harness and a time-to-GPT-2 benchmark on one $8\times$H100 node \citep{karpathy2026nanochat}. Qualification uses the 22-task DCLM CORE metric introduced with DataComp-LM \citep{li2024datacomp}, with the original GPT-2 checkpoint as the reference capability \citep{radford2019gpt2}. We begin from the exact PR~\#830 recipe available at the time of our experiments \citep{zinzi2026pr830} and use it as the base for our matched-endpoint scale-transfer comparisons.

\section{Method}
The model returned from optimization depends on both the saved terminal checkpoints and the rule used to combine them. We first define that rule and analyze its risk for a fixed trajectory. We then ask how changing the rule can change the preferred terminal learning-rate activity.

\subsection{Terminal Shrinkage Averaging}

Consider optimization of a model with $p$ parameters. Let $\theta_t\in\mathbb{R}^p$ denote the iterate at step $t$. At a terminal step $T$, we save $K$ checkpoints separated by $s$ optimizer steps and define their average as
\begin{equation}
    \bar\theta_T := \frac{1}{K}\sum_{j=0}^{K-1}\theta_{T-js}.
    \label{eq:checkpoint-average}
\end{equation}
Terminal Shrinkage Averaging returns
\begin{equation}
    \widehat\theta_T(\alpha) := (1-\alpha)\theta_T+\alpha\bar\theta_T, \qquad \alpha\in[0,1].
    \label{eq:tsa}
\end{equation}
The endpoint $\alpha=0$ returns the raw final iterate, while $\alpha=1$ recovers uniform LAWA. We use LAWA as the averaging endpoint, although the interpolation can use any other averaging estimator in its place. Our goal is to characterize when the optimal model lies strictly between these two endpoints.

Because $\theta_T$ is itself included in $\bar\theta_T$, TSA assigns total weight
\[
    1-\alpha+\frac{\alpha}{K}
\]
to the final checkpoint and weight $\alpha/K$ to each of the other $K-1$ checkpoints.

\paragraph{Local quadratic model.}
To study how the interpolation affects performance, we approximate the loss over the short terminal region visited by the saved checkpoints using a positive-semidefinite quadratic model,
\begin{equation}
    L_{\mathrm{quad}}(\theta) := L(\theta^\star) + \frac{1}{2} (\theta-\theta^\star)^\top H (\theta-\theta^\star), \qquad H\succeq0,
    \label{eq:quadraticloss}
\end{equation}
where $\theta^\star$ is a minimizer of the local model. Since $H$ is positive semidefinite, this optimum may not be unique. This approximation is intended to describe the short late-training window used by TSA rather than the full optimization trajectory. Recent work finds that local linear and quadratic models become particularly accurate late in language-model pretraining \citep{meterez2026defensequadraticmodel}.

For $x,y\in\mathbb{R}^p$, define the curvature-weighted inner product and associated seminorm
\begin{equation}
    \langle x,y\rangle_H := x^\top H y, \qquad \|x\|_H^2 := x^\top Hx.
    \label{eq:h-geometry}
\end{equation}

\paragraph{Mean trajectory and stochastic residuals.}
To separate progress along the terminal trajectory from variation across training runs, we decompose each checkpoint into its mean and a stochastic residual. For notational convenience, write
\[
    \theta_j:=\theta_{T-js}, \qquad j=0,\ldots,K-1,
\]
so that $\theta_0=\theta_T$. Define
\begin{equation}
    \mu_j := \E[\theta_j], \qquad \epsilon_j := \theta_j-\mu_j, \qquad \E[\epsilon_j]=0.
\end{equation}
The expectation is over the training stochasticity, including the order of minibatches.

Define the average of the expected checkpoints and the average stochastic residual as
\begin{equation}
    \bar\mu := \frac{1}{K}\sum_{j=0}^{K-1}\mu_j, \qquad \bar\epsilon := \frac{1}{K}\sum_{j=0}^{K-1}\epsilon_j.
    \label{eq:mean-checkpoint-components}
\end{equation}
We also define
\begin{equation}
    m := \mu_0-\theta^\star, \qquad A := \mu_0-\bar\mu.
    \label{eq:m-and-a}
\end{equation}
$m$ is the displacement of the mean of the raw final iterate from the minimizer of the local quadratic model. $A$ is the displacement between the mean of the raw final iterate and the expected checkpoint average.

Substituting these definitions into Eq.~\ref{eq:tsa} gives the exact decomposition
\begin{align}
    \widehat\theta_T(\alpha)-\theta^\star
    &= (1-\alpha)(\mu_0+\epsilon_0) + \alpha(\bar\mu+\bar\epsilon) - \theta^\star \\
    &= m-\alpha A+\epsilon_0+\alpha q,
    \label{eq:tsa-error-decomposition}
\end{align}
where
\begin{equation}
    q := \bar\epsilon-\epsilon_0.
    \label{eq:q-definition}
\end{equation}
Equation~\ref{eq:tsa-error-decomposition} separates the effect of TSA into two parts. The deterministic term $-\alpha A$ describes how averaging moves the mean output model relative to the local minimizer. The stochastic term $\alpha q$ describes how averaging changes the random deviation around that mean.

We measure performance under the local quadratic surrogate through
\begin{equation}
    \mathcal R(\alpha) := \E\!\left[ L_{\mathrm{quad}} \bigl(\widehat\theta_T(\alpha)\bigr) - L_{\mathrm{quad}}(\theta^\star) \right] = \frac{1}{2} \E\left\| \widehat\theta_T(\alpha)-\theta^\star \right\|_H^2.
    \label{eq:tsa-risk-definition}
\end{equation}
Although $\theta^\star$ need not be unique when $H\succeq0$, this quantity is invariant to the choice of local minimizer because distinct minimizers differ only along directions in the null space of $H$.

\begin{theorem}[Risk of Terminal Shrinkage Averaging]
\label[theorem]{thm:tsa-risk}
Under the local quadratic model, let
\[
    m := \mu_0-\theta^\star, \qquad A := \mu_0-\bar\mu, \qquad q := \bar\epsilon-\epsilon_0,
\]
and define
\[
    V_q := \E\|q\|_H^2, \qquad C := \E\langle\epsilon_0,q\rangle_H.
\]
Then the expected excess risk of TSA satisfies
\begin{equation}
    \mathcal R(\alpha) = \mathcal R(0) + \alpha\bigl( -\langle m,A\rangle_H + C \bigr) + \frac{\alpha^2}{2} \left( \|A\|_H^2+V_q \right).
    \label{eq:tsa-risk-compact}
\end{equation}
Whenever $\|A\|_H^2+V_q>0$, this is a convex quadratic in $\alpha$ and
\begin{equation}
    \alpha^\star = \Pi_{[0,1]} \left( \frac{ \langle m,A\rangle_H-C }{ \|A\|_H^2+V_q } \right).
\label{eq:alpha-star}
\end{equation}
\end{theorem}

The linear term in \Cref{eq:tsa-risk-compact} exposes the central tradeoff. Since $A$ points from the average of the saved checkpoints toward the raw final iterate, moving toward the checkpoint average can incur \emph{lag} when the optimizer is still making useful progress. Define
\[
    G_{\mathrm{lag}} := -\langle m,A\rangle_H.
\]
At the same time, averaging can remove stochastic error from the raw final iterate. Define
\[
    G_{\mathrm{noise}} := -C.
\]
Then
\begin{equation}
    \mathcal R(\alpha) = \mathcal R(0) + \alpha \left( G_{\mathrm{lag}}-G_{\mathrm{noise}} \right) + \frac{\alpha^2}{2} \left( \|A\|_H^2+V_q \right).
    \label{eq:risk-lag-noise}
\end{equation}

A small amount of averaging therefore helps when $G_{\mathrm{noise}}>G_{\mathrm{lag}}$: the stochastic error removed by averaging outweighs the cost of moving toward older checkpoints. Full averaging need not be optimal because the quadratic term makes the marginal benefit of increasing $\alpha$ decrease. In particular, an interior optimum occurs when
\[
    0 < G_{\mathrm{noise}}-G_{\mathrm{lag}} < \|A\|_H^2+V_q.
\]
This gives the intuition behind TSA: retain enough weight on the raw final iterate to limit checkpoint lag while averaging strongly enough to reduce terminal variation.

The proof in \Cref{app:tsa-proof} formalizes this fixed-trajectory result, and \Cref{app:quadratic-response} validates empirically that loss varies quadratically in $\alpha$. The next question is whether the same tradeoff changes which trajectory should be trained.

\subsection{A sufficient condition for a more active terminal schedule}
\label{sec:theory-schedule}

The previous analysis holds the training trajectory fixed and changes only the returned estimator. We next ask how changing the estimator can change which terminal trajectory is preferable.

Let $\rho$ parameterize terminal optimizer activity; in our experiments, $\rho$ corresponds to the terminal learning-rate floor. For fixed $\alpha$, write the local risk as
\begin{equation}
    \mathcal R_\alpha(\rho) = B_\alpha(\rho) + S_\alpha(\rho),
\end{equation}
where
\begin{align}
    B_\alpha(\rho) &:= \frac{1}{2} \|m_\rho-\alpha A_\rho\|_H^2,\\
    S_\alpha(\rho) &:= \frac{1}{2} \E\| \epsilon_{0,\rho}+\alpha q_\rho \|_H^2.
\end{align}
The first term captures mean displacement, while the second captures stochastic variation.

Let
\[
    \rho_{\mathrm{raw}} := \arg\min_\rho \mathcal R_0(\rho)
\]
denote the terminal activity preferred when the raw final iterate is returned. $\mathcal R_0'(\rho_{\mathrm{raw}})=0$. Therefore,
\begin{equation}
    \mathcal R_\alpha'(\rho_{\mathrm{raw}}) = \frac{d}{d\rho} \left[ B_\alpha(\rho)-B_0(\rho) \right]_{\rho=\rho_{\mathrm{raw}}} - \frac{d}{d\rho} \left[ S_0(\rho)-S_\alpha(\rho) \right]_{\rho=\rho_{\mathrm{raw}}}.
    \label{eq:schedule-shift}
\end{equation}

Consequently, if the stochastic benefit supplied by averaging grows with terminal activity faster than its additional deterministic cost, then $\mathcal R_\alpha'(\rho_{\mathrm{raw}})<0$. Under local convexity, the risk-minimizing averaged trajectory therefore lies at a larger value of $\rho$.

\paragraph{A one-dimensional corollary.}
We can demonstrate the schedule shift for a simple model using \Cref{eq:schedule-shift}. Consider one noisy terminal update on $L(x)=\tfrac{h}{2}x^2$ and construct TSA from the pre-update and post-update checkpoints. Whenever the risk-minimizing step sizes are interior, the preferred step size under shrinkage strength $\alpha$ is
\begin{equation}
    \eta_\alpha^\star = \frac{\eta_{\mathrm{raw}}^\star}{1-\alpha/2}.
    \label{eq:one-step-eta-shift}
\end{equation}
Thus every $\alpha>0$ strictly increases the preferred terminal step size, and full averaging of the two checkpoints doubles it. The exact risk and proof are provided in \Cref{app:one-step-quadratic}.

This does not assert that increasing terminal activity is always beneficial under averaging, nor does it predict the magnitude of the optimal shift. Rather, it identifies the condition under which changing the returned estimator changes the preferred terminal schedule. The experiments in \Cref{sec:schedule-estimator-interaction} test this interaction directly.

\section{Experiments}

We design the experiments as a sequence of controlled intervention tests. We first ask whether partial averaging improves the output model under a fixed training trajectory. We then ask whether increasing terminal learning-rate activity helps without averaging. Then, we show that our method is robust to choice of optimizer by replicating our results on AdamW. Finally, we combine the interventions to test the central prediction of our theory: an output estimator that reduces terminal variance can prefer a more active learning-rate schedule than the raw final iterate.

We use depth-12 NanoChat models for controlled ablations and a depth-22 model for a final time-to-GPT-2 experiment to demonstrate that our method scales. Unless otherwise stated, all depth-12 comparisons use the same model initialization, optimizer, training batches, and paired validation data so that differences can be attributed to the intervention being studied. We report validation bits per byte (BPB), where lower is better. For the depth-22 speedrun, we additionally report DCLM CORE, for which the qualification threshold is $0.256525$. Full details are provided in \Cref{app:experimental-details}.

\subsection{Partial shrinkage has a replicated interior optimum}
\label{sec:averaging-only}

We first hold each training trajectory fixed and vary only the returned estimator. To test whether the single-trajectory response from our preliminary sweep replicates, we evaluate a dense $\alpha$ grid on five independently ordered confirmation trajectories at the nanochat baseline $5\%$ floor. Every value of $\alpha$ is evaluated on the same held-out batches within a trajectory. The grid is evaluated only after the confirmation estimator has been frozen and is not used to select that rule.

\begin{figure}[h]
    \centering
    \includegraphics[width=0.60\linewidth]{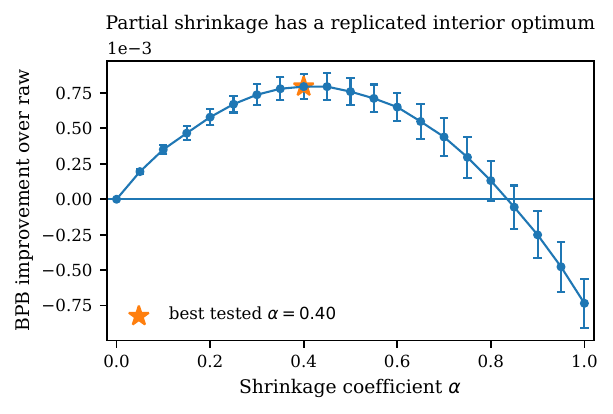}
    \caption{\textbf{Partial shrinkage has a replicated interior optimum.} Points show mean paired BPB improvement over the raw final iterate across five independently ordered depth-12 trajectories at the $5\%$ terminal floor; bars are 95\% $t$-intervals across trajectories. The best tested point is $\alpha=0.40$, whereas full LAWA ($\alpha=1$) is worse than the raw final iterate.}
    \label{fig:d12-averaging}
\end{figure}

The response is broad and clearly interior. The best tested coefficient, $\alpha=0.40$, improves BPB by $0.000794\pm0.000087$; the predeclared confirmation setting $\alpha=0.55$ improves BPB by $0.000711\pm0.000098$. In contrast, uniform LAWA worsens BPB by $0.000734\pm0.000172$. Thus, the benefit of partial rather than full averaging persists across training streams and does not require precise tuning of $\alpha$. The full response at all three terminal floors is reported in \Cref{app:alpha-by-floor} and we compare our method to other common forms of averaging in \Cref{app:additional-d12}.

\subsection{More terminal activity worsens the raw final iterate}
\label{sec:floor-only}

We next vary only the terminal learning-rate floor and always return the raw final iterate. Using the same five paired stream seeds, increasing the floor from $5\%$ to $10\%$ worsens BPB for the raw final iterate by $0.000915\pm0.000836$, while increasing it to $15\%$ worsens BPB for the raw final iterate by $0.003757\pm0.000769$. All five paired differences are positive for both interventions.

\begin{figure}[h]
    \centering
    \includegraphics[width=0.5\linewidth]
    {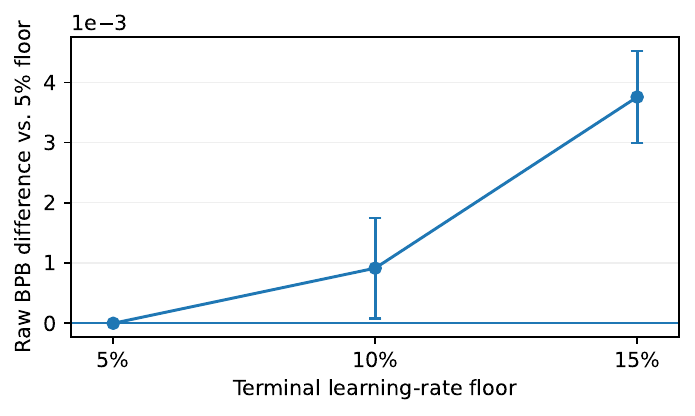}
    \caption{\textbf{More terminal activity worsens the raw final iterate.}
    Points show the mean change in BPB for the raw final iterate relative to the $5\%$ floor across the five streams; bars are 95\% $t$-intervals across streams. The $5\%$ point is zero by construction. Increasing terminal activity degrades the raw final iterate at both tested intervention strengths, with a substantially larger effect at the $15\%$ floor.}
    \label{fig:raw-floor-replicated}
\end{figure}

A denser preliminary single-trajectory sweep through $17.5\%$ shows the same monotone, increasingly steep degradation of the raw final iterate. We report that descriptive sweep separately in \Cref{app:full-preliminary-sweeps}.

These comparisons establish the cost of a more active terminal schedule when the returned model is the raw final iterate. We next test whether TSA changes that cost.

\subsection{The output estimator changes the effect of terminal activity}
\label{sec:schedule-estimator-interaction}

We now vary the schedule and estimator together. We hold TSA fixed at $\alpha=0.55$, $K=8$, and $s=32$, and compare the $5\%$ baseline floor with both $10\%$ and $15\%$ floors over five paired data-order repetitions. For an active floor $\rho$, define the paired interaction
\begin{equation}
    \mathcal I(\rho) := \bigl[L_{\mathrm{raw}}(\rho)-L_{\mathrm{raw}}(5\%)\bigr] - \bigl[L_{\mathrm{TSA}}(\rho)-L_{\mathrm{TSA}}(5\%)\bigr].
    \label{eq:empirical-interaction}
\end{equation}
Positive values mean the active schedule is more favorable under TSA than under the raw final iterate.

\begin{figure}[h]
    \centering
    \includegraphics[width=\linewidth]{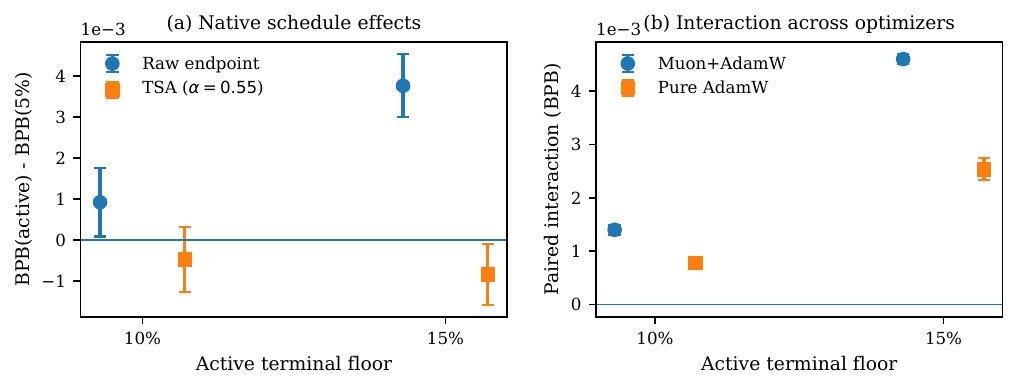}
    \caption{\textbf{The schedule--estimator interaction replicates across intervention strengths and optimizers.} \textbf{Left:} under NanoChat's native Muon+AdamW optimizer, points show the BPB effect of an active floor relative to the $5\%$ floor; positive values mean worse BPB. The $15\%$ floor worsens the raw final iterate but improves TSA. \textbf{Right:} paired interactions from \Cref{eq:empirical-interaction} under the native optimizer and pure AdamW, using the same fixed $\alpha=0.55$. Positive values mean the active floor is more favorable under TSA. Bars are 95\% $t$-intervals over five paired streams.}
    \label{fig:d12-interaction}
\end{figure}

\begin{table}[h]
    \centering
    \caption{\textbf{Two replicated schedule--estimator interventions.} Effects are active-floor BPB minus $5\%$-floor BPB, so positive simple effects indicate worse BPB. The interaction is the effect for the raw final iterate minus that for TSA. Intervals are 95\% $t$-intervals over five paired repetitions.}
    \label{tab:d12-interaction}
    \begin{tabular}{lccc}
        \toprule
        Active floor & Raw final iterate & TSA & Interaction $\mathcal I$ \\
        \midrule
        $10\%$ & $+0.000915\pm0.000836$ & $-0.000487\pm0.000791$ & $\mathbf{+0.001402\pm0.000092}$ \\
        $15\%$ & $+0.003757\pm0.000769$ & $-0.000844\pm0.000737$ & $\mathbf{+0.004601\pm0.000091}$ \\
        \bottomrule
    \end{tabular}
\end{table}

The smaller $10\%$ intervention already yields a clear interaction, even though the TSA simple effect is not individually resolved. At $15\%$, the simple effects themselves point in opposite and statistically resolved directions: the raw final iterate worsens, whereas TSA improves. The paired interaction is positive for every stream at both floors, showing that TSA makes a weaker cooldown more beneficial.

\FloatBarrier

\subsection{The interaction persists under pure AdamW}
\label{sec:optimizer-transfer}

The native optimizer assigns Muon to the main matrix parameter block. To test whether that assignment is necessary for the interaction, we repeat the same $5\%/10\%/15\%$ paired design under pure AdamW. We retain the fixed TSA coefficient $\alpha=0.55$ and measure the within-stream interaction. It is $0.000782\pm0.000089$ BPB for the $10\%$ floor and $0.002536\pm0.000204$ BPB for the $15\%$ floor, and is positive in all five repetitions at both floors (\Cref{fig:d12-interaction}, right). The interaction is smaller than under the native optimizer, but persists without Muon updates, demonstrating that  the phenomenon is not specific to Muon+AdamW.

\subsection{Scale transfer to the NanoChat time-to-GPT-2 benchmark}
\label{sec:d22-speedrun}

The depth-12 results establish the interaction under controlled training conditions. We next test whether the combined schedule and estimator intervention of TSA transfers to the depth-22 NanoChat speedrun setting. We start from the PR~\#830 recipe available at the time of our experiments \citep{zinzi2026pr830}, whose learning-rate schedule is calibrated for a data-to-parameter ratio of $9.4$. We compare three configurations at the same physical ratio-$9.0$ endpoint (step 10,172): an early stop of the original ratio-$9.4$ schedule, the public recipe recalibrated directly to ratio $9.0$, and the ratio-$9.0$ schedule with a $15\%$ terminal learning-rate floor and TSA. For the treatment, we chose $\alpha$ based on the best-performing $\alpha$ in the depth-12 experiment at the $15\%$ floor, which is $0.7$ as seen in \Cref{app:additional-d12}.

Across three training repetitions, TSA improves validation BPB from $0.720235\pm0.000735$ to $0.719580\pm0.000562$ relative to the calibrated baseline (\Cref{tab:d22-matched}). The intervention improves BPB by $0.000654\pm0.000184$ in the paired comparison, with a favorable difference in all three seeds. \footnote{Training-time differences mostly reflect throughput variation across allocations; we do not interpret them as an algorithmic speed difference.}

\begin{table}[h]
    \centering
    \small
    \setlength{\tabcolsep}{3pt}

    \caption{\textbf{Matched-endpoint depth-22 comparison over three seeds.}
    All configurations stop at step $10172$. Time is accumulated training time. Values are means $\pm$ 95\% Student-$t$ intervals.}
    \label{tab:d22-matched}
    \begin{tabular}{lccccc}
        \toprule
        Configuration & Calibration & Floor & Time (min) & Val. BPB $\downarrow$ & CORE $\uparrow$ \\
        \midrule
        PR~\#830 early stop
        & $9.4$ & none
        & $78.028 \pm 2.023$
        & $0.721617 \pm 0.000624$
        & $0.2479 \pm 0.0200$ \\
        PR~\#830 direct $9.0$
        & $9.0$ & none
        & $77.999 \pm 1.938$
        & $0.720235 \pm 0.000735$
        & $0.2491 \pm 0.0106$ \\
        \textbf{$15\%$ floor + TSA}
        & $\mathbf{9.0}$ & $\mathbf{15\%}$
        & $\mathbf{77.896 \pm 0.029}$
        & $\mathbf{0.719580 \pm 0.000562}$
        & $\mathbf{0.2579 \pm 0.0159}$ \\
        \bottomrule
    \end{tabular}
\end{table}

 Relative to the calibrated baseline, the seedwise CORE changes are $+0.0033$, $+0.0124$, and $+0.0108$. The NanoChat qualification threshold is $0.256525$. While the average CORE value exceeds this, only one of the three seeds actually surpassed this threshold. For context, PR~\#830 reports $81.835\pm0.138$ minutes over six runs. Because CORE is discrete and noisy, and only one of our three repetitions individually qualifies, we treat this as a preliminary result with respect to an official record.

\section{Conclusion}
Pretraining need not return the raw final iterate. Once the deployed model is allowed to be a function of the terminal trajectory, the trajectory itself can be optimized differently. We study one simple realization of this idea: keeping optimization more active near the end of training while using Terminal Shrinkage Averaging (TSA) to construct the returned model. A local quadratic analysis characterizes the resulting tradeoff between optimization progress and stochastic variation and explains why partial, rather than uniform, averaging can be optimal.

Our controlled depth-12 experiments isolate the components of the method and establish the central schedule-estimator interaction: increasing terminal learning-rate activity can worsen the raw final iterate while improving the averaged output. At depth 22, the combined intervention improves validation BPB over a matched calibrated baseline in all three paired repetitions and yields one qualifying NanoChat time-to-GPT-2 run. More broadly, these results suggest that the optimization schedule and the rule used to construct the deployed model should be treated as a joint design problem rather than optimized independently.

\paragraph{Limitations.}
TSA explores only a small part of the space of possible output estimators and training trajectories. We use a single global shrinkage coefficient, although different layers or parameter tensors may benefit from different amounts of averaging; we do preliminary explorations of this in \Cref{app:adaptive-tsa} and \Cref{app:structured-tsa}. Likewise, we modify terminal optimization through a simple learning-rate floor, but other changes to cooldown shape, duration, optimizer activity, or checkpoint placement could interact differently with the returned estimator. An issue with averaging is the memory required, which the observations in \Cref{app:groupwise-tsa} help mitigate by showing that averaging's benefit is heterogeneous between tensors. Our theoretical analysis uses a local quadratic approximation and is intended to explain the observed late-training tradeoff rather than provide a complete model of neural-network optimization. Finally, our experiments focus on NanoChat language-model pretraining, and establishing how these interactions transfer across architectures, optimizers, datasets, and substantially larger training regimes remains an important direction for future work.

\clearpage
\bibliography{references}

@inproceedings{ajroldi2025,
  title     = {When, Where and Why to Average Weights?},
  author    = {Ajroldi, Niccol\`o and Orvieto, Antonio and Geiping, Jonas},
  booktitle = {International Conference on Machine Learning},
  year      = {2025}
}

@article{jain2018parallelizing,
  title   = {Parallelizing Stochastic Gradient Descent for Least Squares Regression: Mini-batching, Averaging, and Model Misspecification},
  author  = {Jain, Prateek and Kakade, Sham M. and Kidambi, Rahul and Netrapalli, Praneeth and Sidford, Aaron},
  journal = {Journal of Machine Learning Research},
  volume  = {18},
  number  = {223},
  pages   = {1--42},
  year    = {2018}
}

@inproceedings{bergsma2025straight,
  title     = {Straight to Zero: Why Linearly Decaying the Learning Rate to Zero Works Best for {LLM}s},
  author    = {Bergsma, Shane and Dey, Nolan and Gosal, Gurpreet and Gray, Gavia and Soboleva, Daria and Hestness, Joel},
  booktitle = {International Conference on Learning Representations},
  year      = {2025}
}

@inproceedings{brown2020language,
  title     = {Language Models are Few-Shot Learners},
  author    = {Brown, Tom B. and Mann, Benjamin and Ryder, Nick and Subbiah, Melanie and Kaplan, Jared D. and Dhariwal, Prafulla and Neelakantan, Arvind and Shyam, Pranav and Sastry, Girish and Askell, Amanda and Agarwal, Sandhini and Herbert-Voss, Ariel and Krueger, Gretchen and Henighan, Tom and Child, Rewon and Ramesh, Aditya and Ziegler, Daniel M. and Wu, Jeffrey and Winter, Clemens and Hesse, Christopher and Chen, Mark and Sigler, Eric and Litwin, Mateusz and Gray, Scott and Chess, Benjamin and Clark, Jack and Berner, Christopher and McCandlish, Sam and Radford, Alec and Sutskever, Ilya and Amodei, Dario},
  booktitle = {Advances in Neural Information Processing Systems},
  volume    = {33},
  year      = {2020}
}

@article{defazio2024,
  title   = {The Road Less Scheduled},
  author  = {Defazio, Aaron and Yang, Xingyu Alice and Mehta, Harsh and Mishchenko, Konstantin and Khaled, Ahmed and Cutkosky, Ashok},
  journal = {arXiv preprint arXiv:2405.15682},
  year    = {2024}
}

@inproceedings{gupta2018shampoo,
  title     = {Shampoo: Preconditioned Stochastic Tensor Optimization},
  author    = {Gupta, Vineet and Koren, Tomer and Singer, Yoram},
  booktitle = {International Conference on Machine Learning},
  year      = {2018}
}

@inproceedings{tian2026wsm,
  title     = {{WSM}: Decay-Free Learning Rate Schedule via Checkpoint Merging
               for {LLM} Pre-Training},
  author    = {Tian, Changxin and Wang, Jiapeng and Zhao, Qian and Chen, Kunlong
               and Liu, Jia and Liu, Ziqi and Mao, Jiaxin and Zhao, Wayne Xin
               and Zhang, Zhiqiang and Zhou, Jun},
  booktitle = {International Conference on Learning Representations},
  year      = {2026},
  url       = {https://openreview.net/forum?id=HhThhjKyfw}
}

@misc{au2026training,
  title         = {Training for the Model You Return: Improving Optimization
                   for Iterate-Averaged Language Models},
  author        = {Au, Kwok Chun and Block, Adam},
  year          = {2026},
  eprint        = {2606.25086},
  archivePrefix = {arXiv},
  url           = {https://arxiv.org/abs/2606.25086}
}

@inproceedings{hagele2024scaling,
  title     = {Scaling Laws and Compute-Optimal Training Beyond Fixed Training Durations},
  author    = {H{{\"a}}gele, Alexander and Bakouch, Elie and Kosson, Atli and Ben Allal, Loubna and von Werra, Leandro and Jaggi, Martin},
  booktitle = {Advances in Neural Information Processing Systems},
  year      = {2024},
  url       = {https://arxiv.org/abs/2405.18392}
}

@article{hoffmann2022training,
  title   = {Training Compute-Optimal Large Language Models},
  author  = {Hoffmann, Jordan and Borgeaud, Sebastian and Mensch, Arthur and Buchatskaya, Elena and Cai, Trevor and Rutherford, Eliza and de Las Casas, Diego and Hendricks, Lisa Anne and Welbl, Johannes and Clark, Aidan and Hennigan, Tom and Noland, Eric and Millican, Katie and van den Driessche, George and Damoc, Bogdan and Guy, Aurelia and Osindero, Simon and Simonyan, Karen and Elsen, Erich and Rae, Jack W. and Vinyals, Oriol and Sifre, Laurent},
  journal = {arXiv preprint arXiv:2203.15556},
  year    = {2022}
}

@misc{hu2024minicpm,
  title         = {{MiniCPM}: Unveiling the Potential of Small Language Models with Scalable Training Strategies},
  author        = {Hu, Shengding and Tu, Yuge and Han, Xu and He, Chaoqun and Cui, Ganqu and Long, Xiang and Zheng, Zhi and Fang, Yewei and Huang, Yuxiang and Zhao, Weilin and Zhang, Xinrong and Thai, Zheng Leng and Zhang, Kaihuo and Wang, Chongyi and Yao, Yuan and Zhao, Chenyang and Zhou, Jie and Cai, Jie and Zhai, Zhongwu and Ding, Ning and Jia, Chao and Zeng, Guoyang and Li, Dahai and Liu, Zhiyuan and Sun, Maosong},
  year          = {2024},
  eprint        = {2404.06395},
  archivePrefix = {arXiv},
  primaryClass  = {cs.CL},
  url           = {https://arxiv.org/abs/2404.06395}
}

@inproceedings{izmailov2018averaging,
  title     = {Averaging Weights Leads to Wider Optima and Better Generalization},
  author    = {Izmailov, Pavel and Podoprikhin, Dmitrii and Garipov, Timur and Vetrov, Dmitry and Wilson, Andrew Gordon},
  booktitle = {Proceedings of the Thirty-Fourth Conference on Uncertainty in Artificial Intelligence},
  year      = {2018}
}

@misc{jordan2024muon,
  title        = {{Muon}: An Optimizer for Hidden Layers in Neural Networks},
  author       = {Jordan, Keller and Jin, Yuchen and Boza, Vlado and Jiacheng, You and Cesista, Franz and Newhouse, Laker and Bernstein, Jeremy},
  year         = {2024},
  howpublished = {\url{https://github.com/KellerJordan/Muon}}
}

@article{kaplan2020,
  title   = {Scaling Laws for Neural Language Models},
  author  = {Kaplan, Jared and McCandlish, Sam and Henighan, Tom and Brown, Tom B. and Chess, Benjamin and Child, Rewon and Gray, Scott and Radford, Alec and Wu, Jeffrey and Amodei, Dario},
  journal = {arXiv preprint arXiv:2001.08361},
  year    = {2020}
}

@misc{karpathy2026nanochat,
  title        = {nanochat: A Minimal Experimental Harness for Training Language Models},
  author       = {Karpathy, Andrej and contributors},
  year         = {2026},
  howpublished = {\url{https://github.com/karpathy/nanochat}}
}

@article{kaddour2022stop,
  title   = {Stop Wasting My Time! Saving Days of {ImageNet} and {BERT} Training with Latest Weight Averaging},
  author  = {Kaddour, Jean},
  journal = {arXiv preprint arXiv:2209.14981},
  year    = {2022}
}

@inproceedings{kingma2015,
  title     = {Adam: A Method for Stochastic Optimization},
  author    = {Kingma, Diederik P. and Ba, Jimmy},
  booktitle = {International Conference on Learning Representations},
  year      = {2015}
}

@inproceedings{li2024datacomp,
  title     = {{DataComp-LM}: In Search of the Next Generation of Language Model Pretraining Datasets},
  author    = {Li, Jeffrey and others},
  booktitle = {Advances in Neural Information Processing Systems, Datasets and Benchmarks Track},
  year      = {2024},
  note      = {arXiv:2406.11794}
}

@article{liu2025muon,
  title   = {{Muon} is Scalable for {LLM} Training},
  author  = {Liu, Jingyuan and Su, Jianlin and Yao, Xingcheng and Jiang, Zhejun and Lai, Guokun and Du, Yulun and Qin, Yidao and Xu, Weixin and Lu, Enzhe and Yan, Junjie and others},
  journal = {arXiv preprint arXiv:2502.16982},
  year    = {2025}
}

@inproceedings{loshchilov2019,
  title     = {Decoupled Weight Decay Regularization},
  author    = {Loshchilov, Ilya and Hutter, Frank},
  booktitle = {International Conference on Learning Representations},
  year      = {2019}
}

@article{meterez2026anytime,
  title   = {Anytime Pretraining: Horizon-Free Learning-Rate Schedules with Weight Averaging},
  author  = {Meterez, Alexandru and Nair, Pranav Ajit and Morwani, Depen and Pehlevan, Cengiz and Kakade, Sham},
  journal = {arXiv preprint arXiv:2602.03702},
  year    = {2026}
}

@misc{meterez2026defensequadraticmodel,
  title         = {A Defense of the Quadratic Model},
  author        = {Meterez, Alexandru and Nair, Pranav Ajit and Morwani, Depen and Pehlevan, Cengiz and Kakade, Sham and Damian, Alex},
  year          = {2026},
  eprint        = {2607.21716},
  archivePrefix = {arXiv},
  primaryClass  = {cs.LG},
  url           = {https://arxiv.org/abs/2607.21716}
}

@article{neu2018,
  title   = {Iterate Averaging as Regularization for Stochastic Gradient Descent},
  author  = {Neu, Gergely and Rosasco, Lorenzo},
  journal = {arXiv preprint arXiv:1802.08009},
  year    = {2018}
}

@article{polyak1992acceleration,
  title   = {Acceleration of Stochastic Approximation by Averaging},
  author  = {Polyak, Boris T. and Juditsky, Anatoli B.},
  journal = {SIAM Journal on Control and Optimization},
  volume  = {30},
  number  = {4},
  pages   = {838--855},
  year    = {1992},
  doi     = {10.1137/0330046}
}

@techreport{radford2019gpt2,
  title       = {Language Models are Unsupervised Multitask Learners},
  author      = {Radford, Alec and Wu, Jeffrey and Child, Rewon and Luan, David and Amodei, Dario and Sutskever, Ilya},
  institution = {OpenAI},
  year        = {2019}
}

@inproceedings{sanyal2024early,
  title     = {Early Weight Averaging Meets High Learning Rates for {LLM} Pre-Training},
  author    = {Sanyal, Sunny and Neerkaje, Atula Tejaswi and Kaddour, Jean and Kumar, Abhishek and Sanghavi, Sujay},
  booktitle = {First Conference on Language Modeling},
  year      = {2024},
  url       = {https://openreview.net/forum?id=IA8CWtNkUr}
}

@article{touvron2023llama2,
  title   = {Llama 2: Open Foundation and Fine-Tuned Chat Models},
  author  = {Touvron, Hugo and Martin, Louis and Stone, Kevin and Albert, Peter and Almahairi, Amjad and Babaei, Yasmine and Bashlykov, Nikolay and Batra, Soumya and Bhargava, Prajjwal and Bhosale, Shruti and Bikel, Dan and Blecher, Lukas and Ferrer, Cristian Canton and Chen, Moya and Cucurull, Guillem and Esiobu, David and Fernandes, Jude and Fu, Jeremy and Fu, Wenyin and Fuller, Brian and Gao, Cynthia and Goswami, Vedanuj and Goyal, Naman and Hartshorn, Anthony and Hosseini, Saghar and Hou, Rui and Inan, Hakan and Kardas, Marcin and Kerkez, Viktor and Khabsa, Madian and Kloumann, Isabel and Korenev, Artem and Koura, Punit Singh and Lachaux, Marie-Anne and Lavril, Thibaut and Lee, Jenya and Liskovich, Diana and Lu, Yinghai and Mao, Yuning and Martinet, Xavier and Mihaylov, Todor and Mishra, Pushkar and Molybog, Igor and Nie, Yixin and Poulton, Andrew and Reizenstein, Jeremy and Rungta, Rashi and Saladi, Kalyan and Schelten, Alan and Silva, Ruan and Smith, Eric Michael and Subramanian, Ranjan and Tan, Xiaoqing Ellen and Tang, Binh and Taylor, Ross and Williams, Adina and Kuan, Jian Xiang and Xu, Puxin and Yan, Zheng and Zarov, Iliyan and Zhang, Yuchen and Fan, Angela and Kambadur, Melanie and Narang, Sharan and Rodriguez, Aurelien and Stojnic, Robert and Edunov, Sergey and Scialom, Thomas},
  journal = {arXiv preprint arXiv:2307.09288},
  year    = {2023}
}

@article{vyas2024soap,
  title   = {{SOAP}: Improving and Stabilizing Shampoo using Adam},
  author  = {Vyas, Nikhil and Morwani, Depen and Zhao, Rosie and Kwun, Mujin and Shapira, Itai and Brandfonbrener, David and Janson, Lucas and Kakade, Sham},
  journal = {arXiv preprint arXiv:2409.11321},
  year    = {2024}
}

@article{wen2025fantastic,
  title   = {Fantastic Pretraining Optimizers and Where to Find Them},
  author  = {Wen, Kaiyue and Hall, David and Ma, Tengyu and Liang, Percy},
  journal = {arXiv preprint arXiv:2509.02046},
  year    = {2025}
}

@article{wen2026hyperball,
  title   = {Fantastic Pretraining Optimizers and Where to Find Them {II}: Hyperball Optimization},
  author  = {Wen, Kaiyue and Dang, Xingyu and Lyu, Kaifeng and Ma, Tengyu and Liang, Percy},
  journal = {arXiv preprint arXiv:2606.16899},
  year    = {2026}
}

@inproceedings{zhang2019lookahead,
  title     = {Lookahead Optimizer: k Steps Forward, 1 Step Back},
  author    = {Zhang, Michael R. and Lucas, James and Ba, Jimmy and Hinton, Geoffrey E.},
  booktitle = {Advances in Neural Information Processing Systems},
  volume    = {32},
  year      = {2019}
}

@misc{zinzi2026pr830,
  title        = {Speed up the {8xH100} {GPT-2} Run to 81.8 Minutes},
  author       = {Zinzi, Giovanni},
  year         = {2026},
  howpublished = {nanochat pull request \#830, \url{https://github.com/karpathy/nanochat/pull/830}}
}
\bibliographystyle{iclr2027_conference}

\clearpage
\appendix

\section{Experimental details}
\label{app:experimental-details}

This section specifies the training trajectories, estimator selection, evaluation protocol, and scale-transfer comparison underlying the main experiments.

\subsection{Depth-12 model and common training setup}
\label{app:d12-details}

The depth-12 experiments use the NanoChat architecture with 12 transformer layers and 286,261,730 trainable parameters. Runs use sequence length 512, device batch size 4, gradient accumulation 16, and therefore 32,768 tokens per optimizer step. Every trajectory is trained for $T=3000$ optimizer steps, with the first 40 steps used for learning-rate warmup. The model initialization seed is 1337 in every run; training stochasticity is introduced by an explicit permutation of the 3000 stored optimizer-step batch blocks.

The native optimizer follows NanoChat's heterogeneous parameter assignment: Muon updates the main matrix-valued parameter block, while AdamW updates the remaining parameter groups. The corresponding learning-rate settings are matrix learning rate $0.02$, embedding learning rate $0.2$, unembedding learning rate $0.004$, scalar learning rate $0.5$, and weight decay $0.28$. In the pure-AdamW control, the matrix block is instead updated by AdamW with learning rate $0.003$, $\beta_1=0.9$, $\beta_2=0.95$, $\epsilon=10^{-10}$, and matrix weight decay $0.1$; the remaining group-specific settings are inherited from the same pack. We use the pure-AdamW arm only for within-optimizer estimator comparisons, not to rank the two optimizers by absolute BPB.

A terminal floor $\rho$ means that we follow the original cooldown until its learning-rate multiplier reaches $\rho$ and then hold the multiplier fixed for the remainder of training. Unless otherwise stated, TSA uses $K=8$ checkpoints separated by $s=32$ steps. The selected steps end at step 3000 and span 224 optimizer steps, approximately 7.5\% of training. Snapshots are stored in bfloat16 and all returned candidates are materialized and evaluated in float32.

\subsection{Development and confirmation split}
\label{app:development-confirmation}

The new replicated suite separates rule development from confirmation. The development stage contains five streams for each optimizer at the $10\%$ floor, with stream seeds
\[
    \{6103,7203,8303,9403,10503\}.
\]
On a calibration block, it evaluates a scalar grid $\alpha\in\{0,0.05,\ldots,1\}$ and a two-group grid $\alpha_A,\alpha_B\in\{0,0.125,\ldots,1\}$. Rules maximize mean paired BPB gain over the raw final iterate across the five development streams, with deterministic ties favoring less extreme coefficients.

The confirmation stage contains five new stream seeds,
\[
    \{11103,12203,13303,14403,15503\},
\]
and crosses two optimizers with terminal floors $5\%$, $10\%$, and $15\%$, for 30 independently trained trajectories. Within each stream seed, all floor and optimizer arms begin from the same sampled initial parameter vector and use the same explicit data-order permutation. A smoke test also verifies matching sampled training and evaluation data fingerprints across the native and pure-AdamW packs. Confirmation trajectories are trained and saved before the rule-freeze job, but are not evaluated until the frozen-rule file exists.

The native development stage selects scalar $\alpha=0.55$. This is also the pre-existing main-paper setting, and it is used unchanged in both optimizer regimes in \Cref{sec:optimizer-transfer}. Pure-AdamW development separately selects $\alpha=0.50$; optimizer-specific selected results are reported only in the structured-estimator appendix below. The distinction lets the main text test transfer of the same estimator while retaining the full development study for diagnostics.

\subsection{Evaluation protocol and uncertainty}
\label{app:evaluation-details}

Each pack contains 256 stored evaluation batches. The replicated suite reserves batches 0--31 for the earlier curve block, batches 32--95 for coefficient calibration, batches 96--159 for floor-development bookkeeping, and batches 160--255 as the untouched 96-batch confirmation holdout. All estimators constructed from a given trajectory are evaluated on exactly the same holdout batches.

For single-trajectory sweeps, intervals quantify paired variation across held-out evaluation batches and do not measure training-run uncertainty. For the replicated experiments, all intervals in the main text and this appendix are 95\% Student-$t$ intervals across five independently ordered training streams. Estimator gains and schedule--estimator interactions are computed within stream before averaging, preserving the paired design.

The schedule effects for the raw final iterate and TSA are strongly correlated across paired streams ($r=0.995$ for the $5\%\!\to\!10\%$ intervention and $r=0.994$ for $5\%\!\to\!15\%$ under the native optimizer). This common-mode variation is removed by the within-stream interaction, explaining why its uncertainty is substantially smaller than that of either simple schedule effect.

\subsection{Depth-22 scale-transfer protocol}
\label{app:d22-details}

The depth-22 experiments start from NanoChat PR~\#830 at commit \nolinkurl{e09bc164162f35da0b5b8315be791e9e974a4c3a}.
The recipe uses a depth-22 model, a 49,152-token tokenizer, FP8 training, fused Liger cross-entropy, device batch size 32, and learnable RMSNorm scales. A training ratio denotes the target data-to-parameter ratio; ratio $9.0$ corresponds to optimizer step 10,172, whereas the original PR~\#830 ratio-$9.4$ schedule has a horizon of 10,624 steps.

We compare three configurations at the same physical ratio-$9.0$ endpoint. First, the \emph{early-stop} control follows the original ratio-$9.4$ PR~\#830 schedule but stops at step 10,172 and returns the raw final iterate. Second, the \emph{direct-$9.0$} control recalibrates the schedule itself to ratio $9.0$, again with no added terminal floor, and returns the raw final iterate. Third, the treatment uses the same direct ratio-$9.0$ schedule, clamps the terminal learning-rate multiplier from below at $15\%$ of peak, and returns TSA. Each configuration is evaluated for experiment seeds $42$, $43$, and $44$.

For TSA, we use $K=8$ checkpoints at
\[
9381,\ 9494,\ 9607,\ 9720,\ 9833,\ 9946,\ 10059,\ 10172,
\]
corresponding to a spacing of 113 optimizer steps and a total terminal window of 791 steps, approximately $7.8\%$ of the ratio-$9.0$ endpoint horizon. The 113-step spacing was originally obtained by transferring the depth-12 spacing fraction,
\[
    \frac{32}{3000}\approx 1.07\%,
\]
to the ratio-$9.4$ depth-22 schedule horizon. When recalibrating the treatment schedule to ratio $9.0$, we intentionally retain these same physical checkpoint steps rather than rescaling the averaging window. Thus, the direct-$9.0$ control and treatment differ in terminal learning-rate activity and returned estimator, while the TSA checkpoint window remains fixed.

We select the depth-22 TSA coefficient by selecting the best-performing $\alpha$ in the depth-12 experiment for a $15\%$ floor. We therefore used $\alpha=0.70$ from \Cref{app:additional-d12}.

For each configuration, we report canonical validation BPB and DCLM CORE over the three experiment seeds. Training time is NanoChat's accumulated training-iteration time through step 10,172 rather than end-to-end wall-clock duration. TSA introduces no additional optimizer steps. Because the repetitions span separate allocations, small timing differences mostly include hardware-session variation; we therefore use training time primarily to establish the matched compute budget rather than to claim an algorithmic speed difference between configurations.

\clearpage

\section{Additional depth-12 results}
\label{app:additional-d12}

The main text uses fixed scalar TSA to isolate the schedule--estimator interaction. Here we examine the coefficient response across conditions, sensitivity to checkpoint selection, and whether alternative estimators or parameter-specific rules change that picture.

\subsection{Replicated shrinkage responses across floors and optimizers}
\label{app:alpha-by-floor}

After freezing the confirmation estimators, we evaluate the complete scalar $\alpha$ grid on the untouched holdouts. These curves are descriptive and do not alter the predeclared confirmation rule.

\begin{figure}[h]
    \centering
    \includegraphics[width=\linewidth]{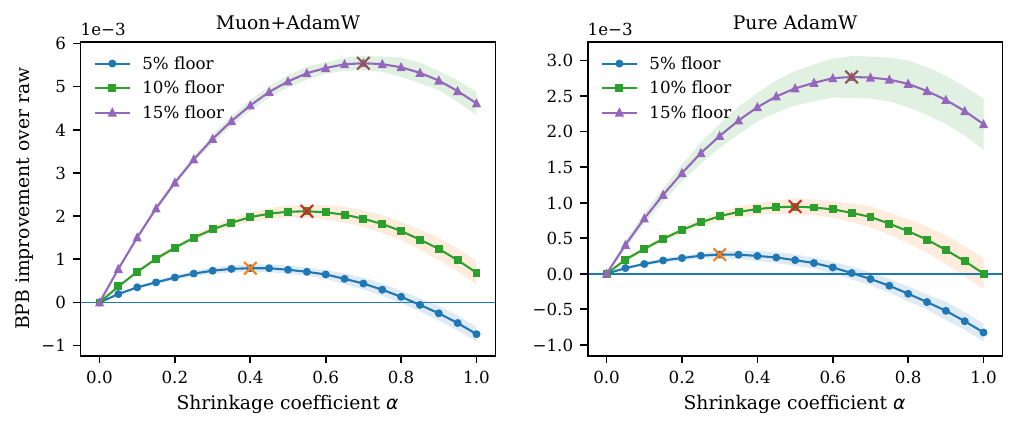}
    \caption{\textbf{The shrinkage response remains interior and shifts with terminal activity.} Curves show mean paired BPB improvement over the raw final iterate across five confirmation trajectories; shaded regions are 95\% $t$-intervals across streams. Crosses mark the best tested coefficients. Under the native optimizer the best tested $\alpha$ moves from $0.40$ to $0.55$ to $0.70$ as the floor rises from $5\%$ to $10\%$ to $15\%$; under pure AdamW it moves from $0.30$ to $0.50$ to $0.65$.}
    \label{fig:app-alpha-by-floor}
\end{figure}

\begin{table}[h]
    \centering
    \caption{\textbf{Descriptive optima of the post-freeze scalar grids.} The grids were evaluated on confirmation holdouts after the estimator rules had been frozen and were not used for selection.}
    \label{tab:app-alpha-optima}
    \begin{tabular}{llcc}
        \toprule
        Optimizer & Floor & Best tested $\alpha$ & Gain over the raw final iterate \\
        \midrule
        Muon+AdamW & 5\%  & 0.40 & $0.000794\pm0.000087$ \\
                   & 10\% & 0.55 & $0.002113\pm0.000154$ \\
                   & 15\% & 0.70 & $0.005543\pm0.000166$ \\
        \midrule
        Pure AdamW & 5\%  & 0.30 & $0.000269\pm0.000053$ \\
                   & 10\% & 0.50 & $0.000945\pm0.000116$ \\
                   & 15\% & 0.65 & $0.002770\pm0.000295$ \\
        \bottomrule
    \end{tabular}
\end{table}

The systematic rightward shift is consistent with the paper's comparative statics: when the terminal trajectory is more active, the returned model can benefit from stronger shrinkage. It also clarifies why a single moderate value such as $\alpha=0.55$ is robust across the tested settings even though it is not the pointwise optimum at every floor. Although the pointwise optimum shifts with terminal activity, the fixed $\alpha=0.55$ rule remains effective across the native-optimizer conditions: it retains approximately 90\%, 100\%, and 96\% of the best observed scalar-TSA gain at the 5\%, 10\%, and 15\% floors, respectively. Thus the interaction does not depend on precise coefficient tuning.

\subsection{Sensitivity to checkpoint count and spacing}
\label{app:ks-sensitivity}

We train five additional depth-12 trajectories with the $10\%$ floor and stream seeds $\{1103,2203,3303,4403,5503\}$. Every estimator is evaluated post hoc on the same five trajectories. We fix $\alpha=0.55$ and vary $K\in\{4,8,16\}$ and $s\in\{16,32,64\}$.

\begin{figure}[h]
    \centering
    \includegraphics[width=0.56\linewidth]{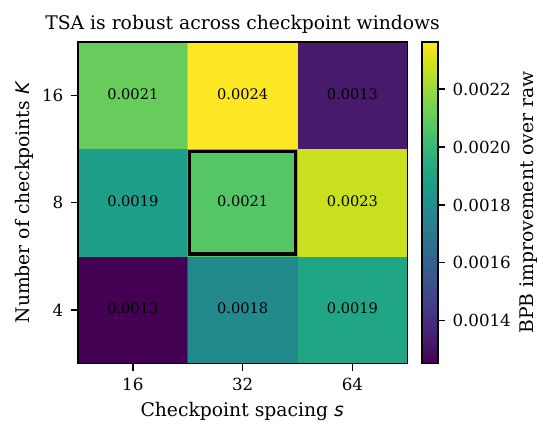}
    \caption{\textbf{TSA remains beneficial across checkpoint counts and spacings.} Cells show mean BPB improvement over the raw final iterate across five training repetitions. The outlined cell is the main-paper setting $K=8,s=32$.}
    \label{fig:app-ks-sensitivity}
\end{figure}

TSA improves over the raw final iterate for all nine windows. The main $K=8,s=32$ setting gains $0.002070\pm0.000205$ BPB; the best tested window, $K=16,s=32$, gains $0.002363\pm0.000207$. Uniform LAWA is more sensitive: at $K=8,s=32$ it gains only $0.000684\pm0.000261$, and at $K=16,s=64$ it is worse than the raw final iterate by $0.005222\pm0.000626$. Partial shrinkage therefore broadens the useful range of checkpoint windows.

These trajectories and their evaluation block are separate from the main confirmation suite, so their nominally identical $K=8,s=32$ estimate need not numerically equal \Cref{tab:app-alpha-optima}.

\subsection{Alternative output estimators}
\label{app:averaging-definitions}

The checkpoint-window results motivate a comparison with other ways to weight recent checkpoints. We use the following returned estimators.

Let $\theta^{(1)},\ldots,\theta^{(K)}$ denote selected checkpoints ordered from oldest to newest, with $\theta^{(K)}=\theta_T$.

\paragraph{Raw final iterate.}
$\widehat\theta_{\mathrm{raw}}=\theta_T$.

\paragraph{Latest Weight Averaging (LAWA).}
\[
    \widehat\theta_{\mathrm{LAWA}} =\frac{1}{K}\sum_{i=1}^{K}\theta^{(i)}.
\]

\paragraph{Terminal Shrinkage Averaging.}
\[
    \widehat\theta_{\mathrm{TSA}}(\alpha) =(1-\alpha)\theta_T+\alpha\widehat\theta_{\mathrm{LAWA}}.
\]
Thus the raw final iterate and LAWA are the endpoints $\alpha=0$ and $\alpha=1$.

\paragraph{Finite-window EWA.}
\[
    \widehat\theta_{\mathrm{EWA}}(\beta) = \frac{\sum_{i=1}^{K}\beta^{K-i}\theta^{(i)}} {\sum_{i=1}^{K}\beta^{K-i}}, \qquad 0<\beta<1.
\]
Matched-window controls use the same $K=8,s=32$ checkpoints and $\beta\in\{0.50,0.75,0.90,0.95\}$.

\paragraph{Checkpoint EMA.}
The checkpoint-EMA control uses the same normalized exponential rule over a denser, longer sequence with $\beta=0.95,K=16,s=16$.

\paragraph{SWA-style late average.}
The SWA-style control is a broader uniform average with $K=32,s=16$. It is mathematically a late checkpoint average but spans a substantially longer terminal window than the main LAWA comparison.

More generally, shrinkage can be applied toward any fixed weighted checkpoint estimator. The same local quadratic decomposition follows after redefining the averaged endpoint. We focus on the uniform finite-window average to keep the method and controlled comparisons simple.

\begin{figure}[h]
    \centering
    \includegraphics[width=0.64\linewidth]{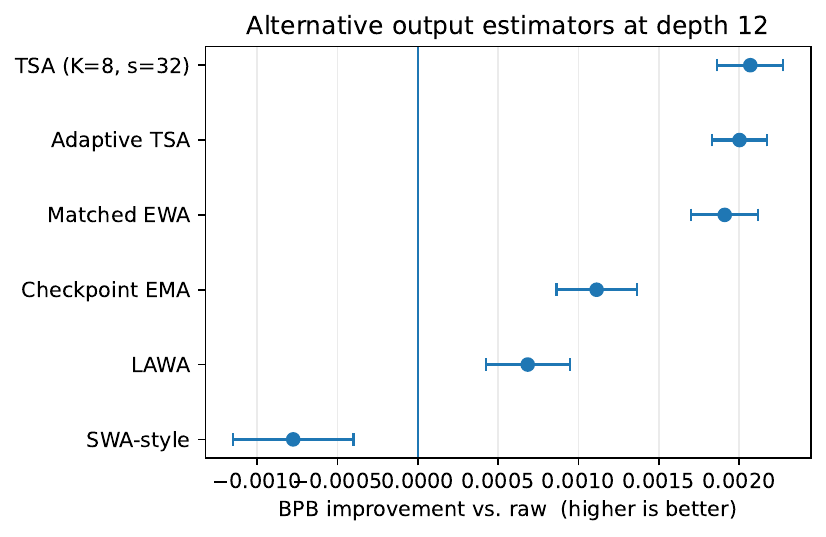}
    \caption{\textbf{Alternative output estimators at depth 12.} Points show mean paired BPB improvement over the raw final iterate, with 95\% $t$-intervals across five training repetitions. Matched EWA denotes the best tested finite-window EWA ($\beta=0.75$) on the same $K=8,s=32$ checkpoints as TSA.}
    \label{fig:app-estimator-family}
\end{figure}

The strongest matched EWA gains $0.001910\pm0.000209$ BPB. TSA gains $0.002070\pm0.000205$, with a paired TSA-over-EWA advantage of $0.000159\pm0.000101$. Checkpoint EMA gains $0.001113\pm0.000250$, while the long SWA-style average is worse than the raw final iterate by $0.000775\pm0.000375$. Recency weighting is useful, but the simple endpoint-shrinkage family remains competitive.

\subsection{Exploratory tensorwise adaptive TSA}
\label{app:adaptive-tsa}

The preceding comparisons use one weighting rule across all parameters, implicitly treating their terminal lag--variation tradeoffs alike. To explore whether terminal updates can inform tensor-specific coefficients, let $\Delta_i^{(g)}=\theta_i^{(g)}-\theta_{i-1}^{(g)}$ be the update between adjacent saved checkpoints for tensor $g$, and let $\bar\Delta^{(g)}$ be the mean update. Define coherent drift energy
\[
    D_g=\|\bar\Delta^{(g)}\|_2^2
\]
and residual update energy $N_g$ after subtracting the mean drift. The empirical noise fraction is
\[
    \nu_g=\frac{N_g}{N_g+D_g+\epsilon}.
\]
As a directional-persistence statistic, let $c_g$ be mean cosine similarity between consecutive updates and define
\[
    o_g=\operatorname{clip}\!\left(1-\max(c_g,0),0,1\right).
\]
The heuristic adaptive rule is
\[
    \alpha_g=\operatorname{clip}\!\left(\frac{\nu_g+o_g}{2},0,1\right), \qquad \widehat\theta_T^{(g)} =(1-\alpha_g)\theta_T^{(g)}+\alpha_g\bar\theta_T^{(g)}.
\]
It is not an estimator of the exact optimum in \Cref{thm:tsa-risk}; it is a trajectory-statistic heuristic inspired by the same lag--variation intuition.

Across the five appendix trajectories, adaptive TSA over all tensors gains $0.002002\pm0.000170$ BPB. Applying the adaptive rule only to the Muon-managed tensors gains $0.001951\pm0.000140$, recovering approximately 97\% of the all-tensor benefit. This observation motivated the initial optimizer-aligned two-group study below.

\subsection{Initial optimizer-aligned structured TSA and parameter localization}
\label{app:structured-tsa}

The adaptive results suggest that shrinkage gains differ across tensors. We test a simpler, predeclared split based on NanoChat's optimizer parameter groups.

Let $P_A$ project onto the flattened coordinates assigned to Muon by NanoChat's native optimizer, and let $P_B=I-P_A$. Group $A$ contains 84,935,088 parameters, or 29.67\% of the model. In pure-AdamW runs we retain the identical coordinate partition even though all parameters are trained with AdamW. Structured TSA is
\begin{equation}
    \widehat\theta_T(\alpha_A,\alpha_B) = \theta_T +\alpha_A P_A(\bar\theta_T-\theta_T) +\alpha_B P_B(\bar\theta_T-\theta_T).
    \label{eq:structured-tsa}
\end{equation}
Scalar TSA is the diagonal restriction $\alpha_A=\alpha_B$.

The development grids select the same off-diagonal pair, $(\alpha_A,\alpha_B)=(0.625,0.250)$, under both optimizers (\Cref{fig:app-structured-development}). The rule and group-schema hashes are written before any confirmation evaluation.

\begin{figure}[h]
    \centering
    \includegraphics[width=\linewidth]{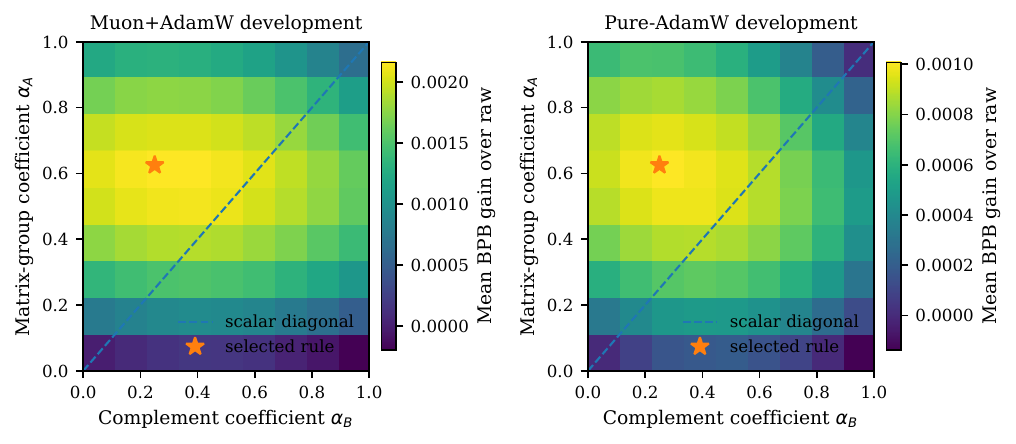}
    \caption{\textbf{Development surfaces for the initial optimizer-aligned two-group TSA.} Each panel reports mean paired BPB gain over the raw final iterate across five development streams at the $10\%$ floor. The dashed line is scalar TSA. The same off-diagonal pair is selected under the native and pure-AdamW optimizers.}
    \label{fig:app-structured-development}
\end{figure}

\begin{table}[h]
    \centering
    \small
    \caption{\textbf{Frozen structured TSA versus frozen scalar TSA.} Development selected $(\alpha_A,\alpha_B)=(0.625,0.250)$ for both optimizers, scalar $\alpha=0.55$ for the native optimizer, and scalar $\alpha=0.50$ for pure AdamW. Positive values in the final column favor structured TSA.}
    \label{tab:app-structured-confirmation}
    \begin{tabular}{llccc}
        \toprule
        Optimizer & Floor & Scalar gain & Structured gain & Structured $-$ scalar \\
        \midrule
        Muon+AdamW
        & 5\%
        & $0.000711\pm0.000098$
        & $0.000730\pm0.000108$
        & $0.000019\pm0.000049$ \\
        & 10\%
        & $0.002113\pm0.000154$
        & $0.002153\pm0.000120$
        & $0.000040\pm0.000050$ \\
        & 15\%
        & $0.005312\pm0.000155$
        & $0.005369\pm0.000117$
        & $0.000057\pm0.000070$ \\
        \midrule
        Pure AdamW
        & 5\%
        & $0.000192\pm0.000081$
        & $0.000243\pm0.000071$
        & $0.000051\pm0.000033$ \\
        & 10\%
        & $0.000945\pm0.000116$
        & $0.000992\pm0.000105$
        & $0.000046\pm0.000017$ \\
        & 15\%
        & $0.002609\pm0.000242$
        & $0.002597\pm0.000243$
        & $-0.000012\pm0.000037$ \\
        \bottomrule
    \end{tabular}
\end{table}

On the native confirmation runs, structured TSA provides only $0.000040\pm0.000050$ BPB beyond the best frozen scalar rule at the predeclared $10\%$ decision gate. The interval includes zero, and the additional gains at $5\%$ and $15\%$ are similarly small. We therefore retain scalar TSA as the main method. Under pure AdamW, the structured increment is positive and small at $5\%$ and $10\%$, but disappears at $15\%$. These findings show measurable heterogeneity without establishing a practically important advantage for the extra parameter.

The group-only ablation is more revealing. Shrinking only the 29.67\% matrix group recovers most of the full TSA gain, whereas shrinking only the much larger complement produces a substantially smaller improvement. This asymmetry persists when all tensors are trained with AdamW (\Cref{fig:app-group-localization,tab:app-group-localization}).

\begin{figure}[h]
    \centering
    \includegraphics[width=\linewidth]{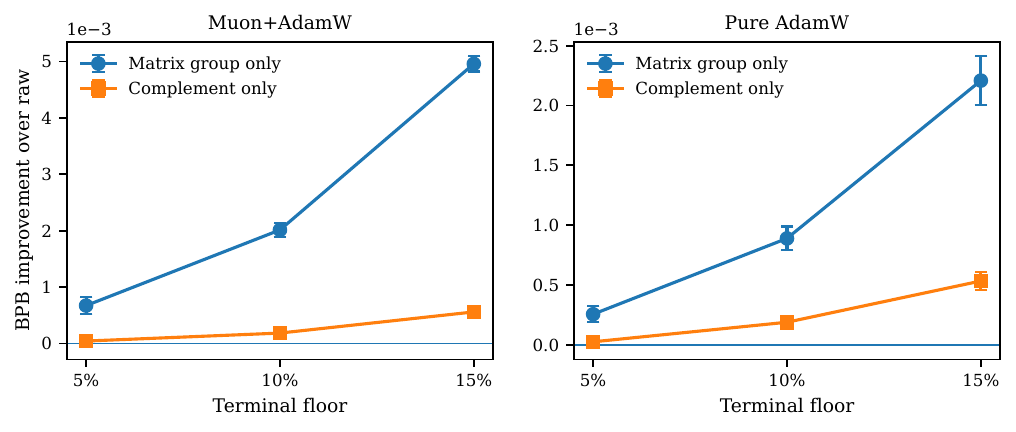}
    \caption{\textbf{The initial two-group split localizes most averaging gain to the matrix parameter group.} Points show paired BPB improvement over the raw final iterate when only one predeclared parameter group is shrunk. The same coordinate partition is used under both optimizers. Bars are 95\% $t$-intervals across five confirmation streams.}
    \label{fig:app-group-localization}
\end{figure}

\begin{table}[h]
    \centering
    \small
    \caption{\textbf{Where the averaging gain is localized.} ``Matrix group'' is the 29.67\% of parameters assigned to Muon in the native optimizer; the identical parameter partition is retained in pure-AdamW runs. All values are paired BPB improvements over the raw final iterate.}
    \label{tab:app-group-localization}
    \begin{tabular}{llccc}
        \toprule
        Optimizer & Floor & Matrix group only & Complement only & Matrix $-$ complement \\
        \midrule
        Muon+AdamW
        & 5\%
        & $0.000673\pm0.000149$
        & $0.000044\pm0.000072$
        & $0.000629\pm0.000192$ \\
        & 10\%
        & $0.002013\pm0.000120$
        & $0.000186\pm0.000049$
        & $0.001828\pm0.000090$ \\
        & 15\%
        & $0.004957\pm0.000136$
        & $0.000560\pm0.000074$
        & $0.004396\pm0.000150$ \\
        \midrule
        Pure AdamW
        & 5\%
        & $0.000255\pm0.000065$
        & $0.000024\pm0.000028$
        & $0.000231\pm0.000056$ \\
        & 10\%
        & $0.000889\pm0.000099$
        & $0.000188\pm0.000032$
        & $0.000701\pm0.000081$ \\
        & 15\%
        & $0.002208\pm0.000204$
        & $0.000533\pm0.000074$
        & $0.001676\pm0.000170$ \\
        \bottomrule
    \end{tabular}
\end{table}

Because the same pattern survives pure AdamW, it cannot be attributed solely to Muon updates. The optimizer-aligned matrix/complement partition is therefore best viewed as a useful first proxy for parameter-role or trajectory heterogeneity rather than as a causal Muon-versus-AdamW explanation. It establishes two facts that motivate the extension below: hidden matrices prefer substantially stronger shrinkage than the complement, while the corresponding two-parameter estimator adds only a small amount beyond scalar TSA.

\FloatBarrier

\subsection{Extending structured TSA to parameter-role groupings}
\label{app:groupwise-tsa}

The optimizer-aligned matrix/complement partition in \Cref{app:structured-tsa} provides a useful first localization: the hidden matrix block prefers substantially stronger shrinkage than the complement, and this asymmetry remains when all parameters are trained with AdamW. At the same time, the corresponding two-parameter estimator improves only modestly over scalar TSA. We therefore retain that split as an explicit baseline and ask whether a broader, but still low-dimensional family of parameter-role partitions explains additional out-of-sample gain.

Equation~\eqref{eq:structured-tsa} is the $G=2$ instance of a general groupwise estimator. Let $\{P_g\}_{g=1}^{G}$ be disjoint coordinate projections satisfying $\sum_{g=1}^{G}P_g=I$. We define
\begin{equation}
    \widehat\theta_T(\boldsymbol\alpha) = \theta_T + \sum_{g=1}^{G} \alpha_g P_g(\bar\theta_T-\theta_T), \qquad \boldsymbol\alpha\in[0,1]^G.
    \label{eq:groupwise-tsa}
\end{equation}
Scalar TSA is the restricted family $\alpha_1=\cdots=\alpha_G$; the matrix/complement rule above is one structured member of this larger family.

\paragraph{Development and confirmation protocol.}
All trajectories in this study use the $10\%$ terminal floor and the main checkpoint window $K=8,s=32$. On eight native Muon+AdamW development trajectories, we evaluated 290 predeclared rules comprising a scalar baseline and 12 structured families. These included the earlier hidden-matrices-versus-rest partition, several functional and geometric subdivisions of the hidden block, trajectory-statistic rules, and partitions that separately expose embedding, unembedding, and remaining parameters. Rules were compared on a 32-batch calibration block. We then froze the scalar baseline, the winner from each structured family, and the overall structured winner before evaluating any new trajectory.

Confirmation used eight entirely fresh native trajectories and eight paired pure-AdamW trajectories, each evaluated on the untouched 96-batch holdout. The same parameter partitions and coefficients were transferred to pure AdamW without retuning. Unless noted otherwise, intervals in this subsection are 95\% Student-$t$ intervals across the eight trajectories.

\begin{figure}[t]
    \centering
    \includegraphics[width=0.98\linewidth]{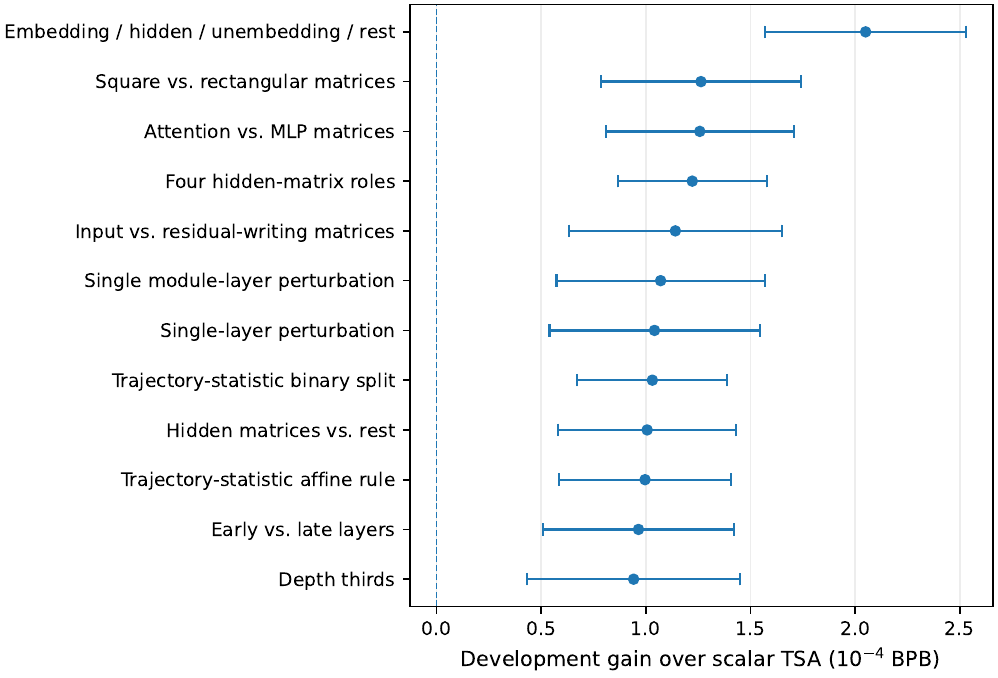}
    \caption{\textbf{Development search over parameter-group structures.} Each point is the best candidate within one predeclared structured family, measured relative to the development-selected scalar TSA rule; bars are 95\% $t$-intervals across eight development trajectories. Because these data were used for selection, the estimates are descriptive rather than confirmatory. The ``Hidden matrices vs. rest'' row is the extension's direct version of the initial optimizer-aligned split in \Cref{app:structured-tsa}; the embedding/hidden/unembedding/rest family produced the overall development winner.}
    \label{fig:app-group-search-development}
\end{figure}

The selected rule separates embedding parameters, hidden transformation matrices, the unembedding matrix, and the remaining scalar/other parameters:
\begin{equation}
    (\alpha_{\mathrm{emb}},\alpha_{\mathrm{hidden}}, \alpha_{\mathrm{unemb}},\alpha_{\mathrm{rest}}) = (0,0.65,0.45,0.25).
    \label{eq:frozen-groupwise-rule}
\end{equation}
The initial localization is therefore preserved rather than overturned: hidden matrices still receive strong shrinkage. The additional gain comes primarily from refining the former complement, especially by separating the input embedding from the unembedding matrix and the remaining parameters.

Because the final checkpoint is included in the $K$-checkpoint average, its total weight in group $g$ is $w_{T,g}=1-(K-1)\alpha_g/K$. For $K=8$, the frozen rule therefore has the interpretation shown in \Cref{tab:app-groupwise-rule}.

\begin{table}[h]
    \centering
    \small
    \caption{\textbf{Frozen parameter-role shrinkage rule.} The final column is the total weight placed on the final checkpoint after accounting for its inclusion in the checkpoint average.}
    \label{tab:app-groupwise-rule}
    \begin{tabular}{lcc}
        \toprule
        Parameter group & $\alpha_g$ & Newest-checkpoint weight $w_{T,g}$ \\
        \midrule
        Embedding parameters         & 0.00 & 1.000 \\
        Hidden matrices              & 0.65 & 0.431 \\
        Unembedding matrix           & 0.45 & 0.606 \\
        Scalars and other parameters & 0.25 & 0.781 \\
        \bottomrule
    \end{tabular}
\end{table}
\FloatBarrier

\paragraph{Fresh confirmation.}
On the eight new native trajectories, the frozen role-based rule improved over the suite-selected scalar baseline $\alpha=0.50$ by $0.000153\pm0.000051$ BPB, with 95\% interval $[0.000103,0.000204]$. It also improved directly over the paper's established scalar setting $\alpha=0.55$ by $0.000138\pm0.000045$ BPB and over the earlier hidden-matrix/rest rule by $0.000108\pm0.000013$ BPB. Relative to the raw final iterate of the same $10\%$-floor trajectories, the selected groupwise estimator gained $0.002316\pm0.000093$ BPB. Thus scalar TSA supplies most of the total averaging benefit, while parameter-role conditioning provides a smaller but resolved additional gain.

The same frozen rule improved over scalar $\alpha=0.50$ by $0.000178\pm0.000023$ BPB under pure AdamW and improved over the hidden/rest rule by $0.000130\pm0.000013$ BPB. This transfer reinforces the conclusion of the initial localization experiment: the signal cannot be explained solely by Muon updates. It is instead consistent with heterogeneity associated with parameter role, tensor geometry, group-specific learning-rate scale, or the resulting terminal dynamics. Pure AdamW does not distinguish among those mechanisms because the parameter groups retain their different base learning rates.

\begin{figure}[t]
    \centering
    \includegraphics[width=0.98\linewidth]{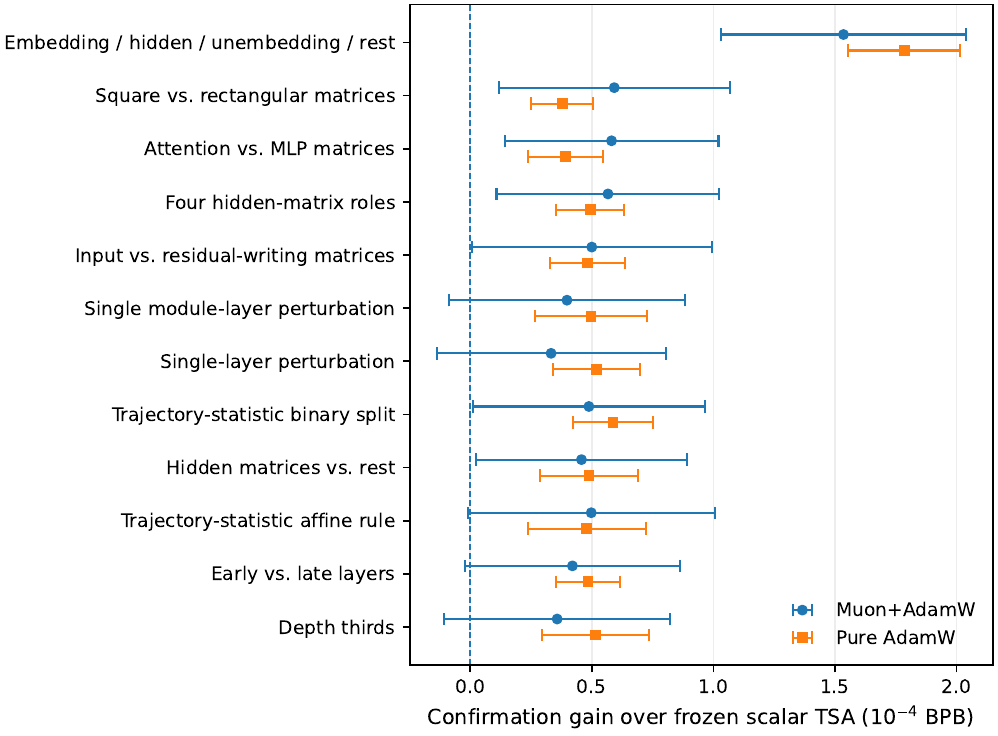}
    \caption{\textbf{Frozen groupwise rules on fresh confirmation trajectories.} Points show the gain of each development-selected family winner over the frozen scalar $\alpha=0.50$ baseline. Circles use eight fresh native Muon+AdamW trajectories; squares transfer the identical rules to eight paired pure-AdamW trajectories without retuning. Bars are 95\% $t$-intervals. The ``Hidden matrices vs. rest'' row carries forward the initial two-group experiment, while the role-based embedding/hidden/unembedding/rest rule provides a further confirmed gain in both optimizer regimes.}
    \label{fig:app-group-search-confirmation}
\end{figure}
\FloatBarrier

Several alternative partitions also carried smaller confirmation signals, including attention-versus-MLP, square-versus-rectangular, and four-way hidden-matrix role splits. Their gains were substantially below that of the selected role-based split. Consistently, the measured trajectory scores of QKV, attention-output, MLP-up, and MLP-down matrices were tightly clustered between $0.708$ and $0.715$, providing little support for a large functional separation within the hidden matrix block. The combined evidence therefore supports a progression rather than a replacement: the initial optimizer-aligned split correctly localized most averaging value to hidden matrices, and the broader search subsequently identified additional heterogeneity within the complement. These experiments do not establish that the selected grouping is universally optimal.

\FloatBarrier

\subsection{Groupwise terminal activity as a mechanism probe}
\label{app:groupwise-floors}

The averaging experiments raise a separate question: if parameter groups prefer different returned estimators, should they also receive different amounts of terminal optimizer activity? We tested this as a fixed mechanism probe rather than a new hyperparameter search. The estimator in \Cref{eq:frozen-groupwise-rule} remained frozen, and its coefficients were mapped monotonically to terminal learning-rate floors via
\begin{equation}
    \rho_g =0.05+0.10\frac{\alpha_g}{0.65}.
    \label{eq:groupwise-floor-map}
\end{equation}
This gives floors of $5.0\%$ for embedding parameters, $15.0\%$ for hidden matrices, $11.92\%$ for the unembedding matrix, and $8.85\%$ for the remaining parameters. The schedule implementation directly routed each optimizer parameter group to one of these four buckets.

We trained eight new paired native trajectories under four fixed schedules: uniform $5\%$, uniform $10\%$, the mapped groupwise schedule, and a uniform floor equal to the parameter-count-weighted mean of the mapped floors ($8.576\%$). Within each stream, all arms shared initialization, data order, checkpoint window, and the 96-batch holdout. For every trajectory, we evaluated the raw final iterate, scalar TSA at $\alpha=0.55$, and the frozen role-based estimator. No schedule or estimator was selected using these trajectories.

For an estimator $E$, define the mapped-schedule effect relative to the parameter-count-matched uniform schedule as $\Delta_E=L_E(\mathrm{mapped})-L_E(\mathrm{matched})$, so positive values mean that the mapped schedule is worse. The corresponding interaction is
\begin{equation}
    \mathcal I_{\mathrm{group}} =\Delta_{\mathrm{raw}}-\Delta_{\mathrm{groupwise}}.
    \label{eq:groupwise-schedule-interaction}
\end{equation}

\begin{table}[h]
    \centering
    \small
    \caption{\textbf{Predeclared groupwise terminal-floor intervention.} Schedule effects are mapped-floor BPB minus the parameter-count-matched uniform-floor BPB; positive values indicate worse BPB. The interaction is the effect for the raw final iterate minus that for groupwise TSA. Intervals are 95\% $t$-intervals over eight paired trajectories.}
    \label{tab:app-groupwise-floor-interaction}
    \begin{tabular}{lcc}
        \toprule
        Returned estimator & Mapped-floor effect & 95\% interval \\
        \midrule
        Raw final iterate
          & $+0.003404\pm0.000386$
          & $[+0.003019,+0.003790]$ \\
        Frozen groupwise TSA
          & $-0.000112\pm0.000294$
          & $[-0.000407,+0.000182]$ \\
        \midrule
        Interaction $\mathcal I_{\mathrm{group}}$
          & $\mathbf{+0.003516\pm0.000098}$
          & $[+0.003418,+0.003614]$ \\
        \bottomrule
    \end{tabular}
\end{table}
\FloatBarrier

The tested heterogeneous schedule therefore did not establish a better final training recipe: under the frozen groupwise estimator, its difference from the matched uniform schedule was small and unresolved. Its effect nevertheless depended strongly on the returned estimator. The mapped allocation made the raw final iterate worse by $0.003404$ BPB, while the groupwise estimator removed essentially all of that relative penalty, yielding the precisely estimated interaction in \Cref{tab:app-groupwise-floor-interaction}.

Parameter-specific averaging also became more valuable on the heterogeneous trajectory. Within the mapped arm, frozen groupwise TSA improved over the raw final iterate by $0.005114\pm0.000174$ BPB and over scalar TSA at $\alpha=0.55$ by $0.000322\pm0.000061$ BPB. The complete mapped-schedule plus groupwise-TSA output improved over the paired uniform-$5\%$ raw final iterate by $0.000991\pm0.000375$ BPB. As an ancillary check, these eight new streams independently reproduced the ordinary uniform-$5\%$ to uniform-$10\%$ schedule--estimator interaction at $0.001474\pm0.000086$ BPB, close to the main confirmation estimate.

The parameter-count-matched control equalizes the mean floor across parameter coordinates, not optimizer-update energy, displacement, or noise injection; the corresponding optimizer groups have different base learning rates and update geometries. We therefore interpret this experiment as evidence that the value of parameter-specific output estimation depends on how terminal activity is distributed across parameter groups, not as evidence that the particular mapped cooldown should replace a uniform schedule.

\subsection{Full preliminary single-trajectory sweeps}
\label{app:full-preliminary-sweeps}

Before the replicated confirmation experiments, we performed two denser diagnostic sweeps on individual depth-12 trajectories. The first holds the training trajectory fixed and varies the scalar TSA coefficient $\alpha$. The second returns the raw final iterate and varies only the terminal learning-rate floor. These experiments provide higher-resolution views of the response surfaces, but their uncertainty reflects evaluation batches within a single trained trajectory rather than variation across independent training runs. We therefore treat them as descriptive preliminary evidence; the corresponding replicated main-text experiments provide the stronger inferential results.

\paragraph{Scalar shrinkage sweep.}
Holding the baseline $5\%$-floor trajectory fixed, partial shrinkage produces a broad interior optimum, while full LAWA eventually becomes worse than the raw final iterate. This preliminary response motivated the replicated estimator sweep in \Cref{fig:d12-averaging}.

\begin{table}[h]
    \centering
    \small
    \caption{\textbf{Preliminary fixed-trajectory scalar sweep.}
    Improvement is BPB for the raw final iterate minus estimator BPB. Intervals quantify paired evaluation-batch uncertainty for one trajectory.}
    \label{tab:app-preliminary-alpha}
    \begin{tabular}{ccc}
        \toprule
        $\alpha$ & Validation BPB $\downarrow$ & Improvement vs. raw final iterate $\uparrow$ \\
        \midrule
        0.00 & 1.061224 & --- \\
        0.10 & 1.060872 & $+0.000352\pm0.000074$ \\
        0.20 & 1.060623 & $+0.000602\pm0.000104$ \\
        0.30 & 1.060483 & $+0.000741\pm0.000138$ \\
        0.40 & 1.060458 & $+0.000766\pm0.000176$ \\
        0.50 & 1.060499 & $+0.000725\pm0.000212$ \\
        0.55 & 1.060554 & $+0.000670\pm0.000225$ \\
        0.60 & 1.060679 & $+0.000546\pm0.000258$ \\
        0.65 & 1.060760 & $+0.000465\pm0.000288$ \\
        0.70 & 1.060852 & $+0.000372\pm0.000299$ \\
        0.75 & 1.060985 & $+0.000239\pm0.000324$ \\
        0.80 & 1.061190 & $+0.000035\pm0.000346$ \\
        0.90 & 1.061570 & $-0.000346\pm0.000373$ \\
        1.00 & 1.062068 & $-0.000844\pm0.000424$ \\
        \bottomrule
    \end{tabular}
\end{table}

\paragraph{Terminal-floor sweep for the raw final iterate.}
We separately vary only the terminal learning-rate floor while returning the raw final iterate. The raw final iterate worsens monotonically over the tested range, and the degradation becomes steeper at the larger floors (\Cref{fig:app-preliminary-floor}). This shape is consistent with the replicated $5\%$, $10\%$, and $15\%$ comparison in \Cref{fig:raw-floor-replicated}.

\begin{figure}[h]
    \centering
    \includegraphics[width=0.68\linewidth]
    {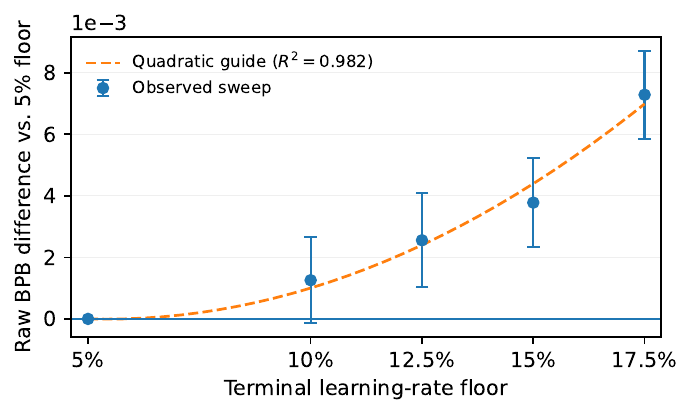}
    \caption{\textbf{Preliminary response of the raw final iterate across terminal floors.}
    Points show BPB differences for the raw final iterate relative to the $5\%$ floor along one training trajectory; bars quantify 95\% paired evaluation-batch uncertainty. The dashed curve is a least-squares quadratic guide, constrained to pass through the $5\%$ baseline ($R^2=0.982$). The fit is descriptive only: unlike the exact quadratic dependence of the local TSA surrogate on shrinkage coefficient $\alpha$, the theory does not imply that BPB for the raw final iterate must be quadratic in the terminal floor.}
    \label{fig:app-preliminary-floor}
\end{figure}

\begin{table}[t]
    \centering
    \small
    \caption{\textbf{Single-trajectory floor sweep without averaging.}
    Differences are relative to the $5\%$ floor. Intervals quantify
    paired evaluation-batch uncertainty for one trajectory.}
    \label{tab:app-preliminary-floor}
    \begin{tabular}{lcc}
        \toprule
        Floor & Validation BPB $\downarrow$ & Difference vs. 5\% $\downarrow$ \\
        \midrule
        5\% & 1.061224 & --- \\
        10\% & 1.062482 & $+0.001257\pm0.001399$ \\
        12.5\% & 1.063779 & $+0.002555\pm0.001529$ \\
        15\% & 1.065003 & $+0.003779\pm0.001450$ \\
        17.5\% & 1.068508 & $+0.007284\pm0.001431$ \\
        \bottomrule
    \end{tabular}
\end{table}

\subsection{Replicated quadratic response}
\label{app:quadratic-response}

Under the local quadratic surrogate, the TSA risk is exactly quadratic in the shrinkage coefficient $\alpha$ for a fixed trajectory. We therefore ask whether the observed validation response has the corresponding shape. For each optimizer and terminal floor, we fit
\begin{equation}
    g(\alpha)=b\alpha+c\alpha^2, \qquad g(0)=0,
    \label{eq:empirical-quadratic-fit}
\end{equation}
to the mean paired BPB improvement over the raw final iterate across the five confirmation trajectories. Because $g$ denotes improvement rather than risk, a beneficial interior optimum corresponds to a concave response. The fitted vertex is $\widehat{\alpha}^{\star}=-b/(2c)$.

\begin{table}[h]
\centering
\small
\caption{\textbf{Quadratic fits to the replicated shrinkage responses.} For each optimizer and terminal floor, the best tested coefficient on the post-freeze scalar grid is compared with the vertex of the fitted quadratic. Fits use the mean paired BPB improvement over the raw final iterate across five confirmation trajectories.}
\label{tab:app-quadratic-fits}
\begin{tabular}{llrrr}
\toprule
Optimizer & Floor & Best tested $\alpha$ & $\widehat{\alpha}^{\star}$ & $R^2$ \\
\midrule
Muon+AdamW
& 5\%  & 0.40 & 0.419 & 0.99977 \\
& 10\% & 0.55 & 0.547 & 0.99976 \\
& 15\% & 0.70 & 0.701 & 0.99970 \\
\midrule
Pure AdamW
& 5\%  & 0.30 & 0.330 & 0.99984 \\
& 10\% & 0.50 & 0.499 & 0.99950 \\
& 15\% & 0.65 & 0.668 & 0.99993 \\
\bottomrule
\end{tabular}
\end{table}

\begin{figure}[h]
    \centering
    \includegraphics[width=0.85\linewidth]{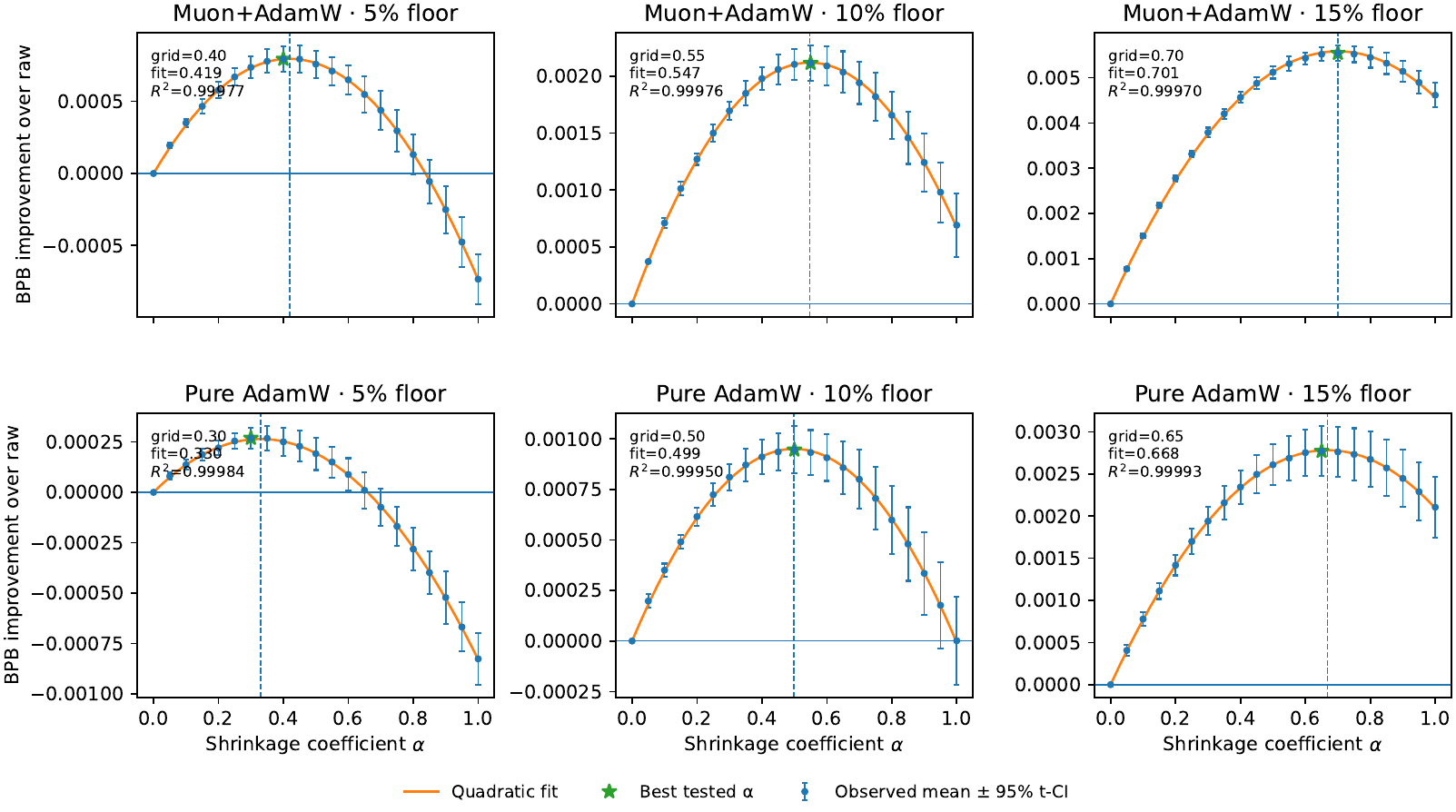}
    \caption{\textbf{The replicated shrinkage responses are nearly quadratic.} Each panel shows mean paired BPB improvement over the raw final iterate across five confirmation trajectories; bars are 95\% $t$-intervals across trajectories. Solid curves are fits of $g(\alpha)=b\alpha+c\alpha^2$ constrained by $g(0)=0$. Stars mark the best tested scalar-grid coefficients, and dashed vertical lines mark the fitted vertices.}
    \label{fig:app-alpha-quadratic}
\end{figure}

The quadratic description is accurate across all six optimizer-floor conditions, with $R^2\geq 0.99950$. The fitted vertices closely track the independently observed grid optima. Under Muon+AdamW, the fitted optimum moves from $0.419$ to $0.547$ to $0.701$ as the terminal floor rises from 5\% to 10\% to 15\%; under pure AdamW it moves from $0.330$ to $0.499$ to $0.668$. Thus the near-quadratic response is not specific to one trajectory, floor, or optimizer regime, and the preferred amount of shrinkage systematically increases with terminal activity. This provides an empirical consistency check for the local quadratic description and its comparative-static interpretation; it does not imply that global neural-network training dynamics are quadratic.

\clearpage
\section{Proof of the TSA risk theorem}
\label{app:tsa-proof}

\begin{proof}[Proof of \Cref{thm:tsa-risk}]
By \Cref{eq:tsa-error-decomposition},
\[
    \widehat\theta_T(\alpha)-\theta^\star =m-\alpha A+\epsilon_0+\alpha q.
\]
The random component $\epsilon_0+\alpha q$ has mean zero. Therefore,
\begin{align}
    2\mathcal R(\alpha) &=\E\left\|m-\alpha A+\epsilon_0+\alpha q\right\|_H^2 \\
    &=\|m-\alpha A\|_H^2+\E\|\epsilon_0+\alpha q\|_H^2 \\
    &=\|m-\alpha A\|_H^2+\E\|\epsilon_0\|_H^2 +2\alpha C+\alpha^2V_q.
\end{align}
Expanding the deterministic term,
\[
    \|m-\alpha A\|_H^2 =\|m\|_H^2-2\alpha\langle m,A\rangle_H+\alpha^2\|A\|_H^2.
\]
Since
\[
    \mathcal R(0) =\frac12\left(\|m\|_H^2+\E\|\epsilon_0\|_H^2\right),
\]
dividing by two yields
\[
    \mathcal R(\alpha) =\mathcal R(0) +\alpha\bigl(-\langle m,A\rangle_H+C\bigr) +\frac{\alpha^2}{2}\bigl(\|A\|_H^2+V_q\bigr).
\]
Differentiating gives
\[
    \mathcal R'(\alpha) =-\langle m,A\rangle_H+C +\alpha\bigl(\|A\|_H^2+V_q\bigr).
\]
When $\|A\|_H^2+V_q>0$, the unconstrained minimizer is
\[
    \alpha_{\mathrm{unc}} =\frac{\langle m,A\rangle_H-C}{\|A\|_H^2+V_q}.
\]
Projection onto $[0,1]$ gives \Cref{eq:alpha-star}. The optimum is strictly interior exactly when
\[
    0<\langle m,A\rangle_H-C<\|A\|_H^2+V_q.
\]
\end{proof}
\clearpage
\section{One-step quadratic analysis}
\label{app:one-step-quadratic}

Consider the one-dimensional quadratic loss
\begin{equation}
    L(x)=L(0)+\frac{h}{2}x^2,
    \qquad h>0.
    \label{eq:one-step-loss}
\end{equation}
Let $x^{-}\neq 0$ be a deterministic parameter value immediately before one terminal stochastic-gradient update, and let
\begin{equation}
    x^{+}=(1-\eta h)x^{-}-\eta\xi,
    \qquad
    \mathbb{E}[\xi]=0,
    \qquad
    \mathbb{E}[\xi^2]=\sigma^2.
    \label{eq:one-step-update}
\end{equation}
Using the two checkpoints $x^{-}$ and $x^{+}$, TSA returns
\begin{equation}
\begin{aligned}
    \widehat{x}_\alpha
    &=(1-\alpha)x^{+}
      +\frac{\alpha}{2}\left(x^{-}+x^{+}\right) \\
    &=\left(1-\frac{\alpha}{2}\right)x^{+}
      +\frac{\alpha}{2}x^{-},
    \qquad \alpha\in[0,1].
\end{aligned}
\label{eq:one-step-tsa}
\end{equation}
Define $a_\alpha:=1-\alpha/2$. Substitution gives
\begin{equation}
    \widehat{x}_\alpha
    =\left(1-a_\alpha\eta h\right)x^{-}
     -a_\alpha\eta\xi.
    \label{eq:one-step-effective-update}
\end{equation}
Consequently, the expected excess quadratic risk is exactly
\begin{equation}
    R_\alpha(\eta)
    :=\mathbb{E}\!\left[L(\widehat{x}_\alpha)-L(0)\right]
    =\frac{h}{2}
      \left[
        \left(1-a_\alpha\eta h\right)^2(x^{-})^2
        +a_\alpha^2\eta^2\sigma^2
      \right].
    \label{eq:one-step-risk}
\end{equation}

\paragraph{Corollary.}
Suppose $\eta$ is selected from $[0,\eta_{\max}]$. Then
\begin{equation}
    \eta_\alpha^\star
    =\Pi_{[0,\eta_{\max}]}
      \left(
        \frac{h(x^{-})^2}
        {a_\alpha\left(h^2(x^{-})^2+\sigma^2\right)}
      \right).
    \label{eq:one-step-optimal-eta}
\end{equation}
If the optima for the raw final iterate and TSA are not clipped by $\eta_{\max}$, then
\begin{equation}
    \eta_\alpha^\star
    =\frac{\eta_{0}^\star}{1-\alpha/2}.
    \label{eq:one-step-optimal-eta-ratio}
\end{equation}
In particular, $\eta_\alpha^\star>\eta_0^\star$ for every $\alpha\in(0,1]$, and $\eta_1^\star=2\eta_0^\star$.

\paragraph{Proof.}
Differentiating \eqref{eq:one-step-risk} yields
\begin{equation}
    R_\alpha'(\eta)
    =h a_\alpha
      \left[
        -h(x^{-})^2 +a_\alpha\eta\left(h^2(x^{-})^2+\sigma^2\right)
      \right],
\end{equation}
while
\begin{equation}
    R_\alpha''(\eta)
    =h a_\alpha^2
      \left(h^2(x^{-})^2+\sigma^2\right)>0.
\end{equation}
The unconstrained minimizer is therefore
\begin{equation}
    \eta_{\alpha,\mathrm{unc}}^\star
    =\frac{h(x^{-})^2}
      {a_\alpha\left(h^2(x^{-})^2+\sigma^2\right)}.
\end{equation}
Projection onto $[0,\eta_{\max}]$ gives \eqref{eq:one-step-optimal-eta}. Since $a_0=1$ and $a_\alpha=1-\alpha/2$, the interior ratio follows. \hfill$\square$

\paragraph{Remark.}
In this minimal one-update model, $a_\alpha\eta_\alpha^\star$ is constant in $\alpha$, so the jointly optimized minimum risk is also constant. The result therefore isolates a change in the preferred schedule rather than claiming an automatic improvement in the best attainable risk. Multi-step trajectories, longer checkpoint windows, and nonstationary late-training dynamics can break this exact rescaling equivalence.

\clearpage
\section{LLM disclosure}
We used LLMs in three ways. First, essentially all experiments were set up and implemented with LLM coding, followed by a separate LLM review of code correctness. Second, we asked LLMs to critique manuscript drafts and incorporated suggestions when we judged them appropriate. Third, we used LLMs to find related papers; we read each paper before citing it to confirm its relation to our work.
\end{document}